\documentclass{article}
\usepackage{blindtext}
\usepackage[a4paper, total={6in, 8in}]{geometry}

\usepackage{graphicx}
\usepackage{amsmath}
\usepackage{amsthm}
\usepackage{booktabs}
\usepackage{algorithm}
\usepackage{algpseudocode}
\usepackage{adjustbox}
\usepackage{amssymb}
\usepackage{amsthm}
\usepackage{blindtext}
\usepackage{enumitem}
\usepackage{multirow}
\usepackage{subcaption}
\usepackage{graphicx}
\theoremstyle{definition}
\newtheorem{assumption}{Assumption}
\newtheorem{theorem}{Theorem}
\newtheorem{definition}{Definition}

\newtheorem{lemma}{Lemma}
\newtheorem{remark}{Remark}

\usepackage{graphicx} 
\usepackage{microtype}
\usepackage{color}
\usepackage{xcolor}         

\begin{document}

\title{ Differentially Private Stochastic Convex Optimization in  (Non)-Euclidean Space Revisited}

\author{Jinyan Su$^1$\footnote{Part of the work was done when Jinyan Su was a research intern at KAUST and The Chinese University of Hong Kong.} \and Changhong Zhao$^{2}$ \and Di Wang$^{3,4,5}$}
\date{
$^1$Mohamed bin Zayed University of Artificial Intelligence\\ 
$^2$Department of Information Engineering,\\
The Chinese University of Hong Kong\\ 
  $^3$Provable Responsible AI and Data Analytics Lab \\
    $^4$Computational Bioscience Research Center\\
   $^5$Division of CEMSE, 
    King Abdullah University of Science and Technology 
}
\maketitle
\begin{abstract}
    In this paper, we revisit the problem of Differentially Private Stochastic Convex Optimization (DP-SCO) in Euclidean and general $\ell_p^d$ spaces. Specifically, we focus on three settings that are still far from well understood:  (1) DP-SCO over a  constrained and bounded (convex) set in Euclidean space; (2) unconstrained DP-SCO in $\ell_p^d$ space; (3) DP-SCO with heavy-tailed data over a  constrained and bounded set in $\ell_p^d$ space. For problem (1), for both convex and strongly convex loss functions, we propose methods whose outputs could achieve (expected) excess population risks that are only dependent on the Gaussian width of the constraint set,  rather than the dimension of the space. Moreover, we also show the bound for strongly convex functions is optimal up to a logarithmic factor. For problems (2) and (3), we propose several novel  algorithms and provide the first theoretical results for both cases when $1<p<2$ and $2\leq p\leq \infty$. 
\end{abstract}

\section{Introduction}
Learning from data that contains sensitive information has become a critical consideration.  It enforces machine learning algorithms to not only learn effectively from the training data but also provide a certain level of guarantee on privacy preservation. To address the privacy concern, as a rigorous notion for statistical data privacy, differential privacy (DP) \cite{dwork2006calibrating} has received much attention in the past few years and has become a de facto technique for private data analysis.

As the two most fundamental models in machine learning, Stochastic Convex Optimization (SCO) \cite{vapnik1999nature} with its empirical form, Empirical Risk Minimization (ERM), 
can find numerous applications, such as biomedicine and healthcare. However, as these applications always involve sensitive data, it is essential to design DP algorithms for SCO and ERM, which corresponds to the problem of DP-SCO and DP-ERM, respectively. DP-SCO and DP-ERM have been
extensively studied  for over a decade, starting from \cite{chaudhuri2008privacy}. For example, \cite{bassily2014private} presents the optimal rates of general DP-ERM for both convex and strongly loss functions. \cite{bassily2019private,feldman2020private} later study the optimal rates of general DP-SCO, which is later extended by \cite{su2022faster,asi2021adapting} to loss functions that satisfy the growth condition. \cite{bassily2021non,asi2021private} provide the first study on DP-SCO over non-Euclidean space, i.e., the $\ell_p$ space with $1\leq p\leq \infty$.

While there are a vast number of studies on DP-SCO/DP-ERM, there are still several open problems left, especially the constrained case in Euclidean space where the convex constraint set has some specific geometric structures, and the case where the space is non-Euclidean. In detail, while it has been shown that the optimal rate of DP-ERM over $\ell_2$-norm ball depends on $O(\sqrt{d})$ and $O(d)$ for convex and strongly convex loss, respectively \cite{bassily2014private}, \cite{talwar2014private}  show that for general constraint set $\mathcal{C}$, the bound on $d$ could be improved  to  $O(G_\mathcal{C})$ and $O(G^2_\mathcal{C})$ for these two classes of functions, where $G_\mathcal{C}$ is the Gaussian width of set $\mathcal{C}$ (see Definition \ref{def:12} for details), which could be far less than the dimension $d$. However, compared to DP-ERM with Gaussian width, DP-SCO with Gaussian width is far from well understood. The best-known result even cannot recover the optimal rate of the $\ell_2$-norm ball case \cite{amid2022public}.  For the non-Euclidean case, \cite{bassily2021non} only study the constrained case where the constrained set has a bounded diameter. Theoretical behaviors for the unconstrained case are still unknown. Moreover, In the Euclidean case, recently, there has been a line of work focusing on DP-SCO where the distribution of loss gradients is heavy-tailed rather than uniformly bounded \cite{wang2020differentially,hu2022high,kamath2022improved}. However, non-Euclidean DP-SCO with heavy-tailed data has not been studied so far. 

In this paper, we study the theoretical behaviors of three problems: (1) DP-SCO (with Lipschitz loss) over  a convex constraint set $\mathcal{C}$ in Euclidean space; (2) unconstrained DP-SCO in $\ell_p^d$ space; (3) DP-SCO with heavy-tailed data over a convex constraint set $\mathcal{C}$ in $\ell_p^d$ space. Specifically, our contributions can be summarized as follows. 

 \noindent {\bf 1.} For problem (1), we consider both convex and strongly convex (smooth) loss functions. We show that for convex functions, there is an $(\epsilon, \delta)$-DP algorithm whose output could achieve an (expected) excess population risk of $O(\frac{ G_{\mathcal{C}}\sqrt{\log(1/\delta)}}{\epsilon n}+\frac{1}{\sqrt{n}})$, where $n$ is the sample size. The rate could be improved to $O(\frac{ G^2_{\mathcal{C}}{\log(1/\delta)}}{ n^2\epsilon^2}+\frac{1}{n})$ for strongly convex functions. Moreover, we also show that the bound for strongly convex functions is optimal up to a factor of $\text{Poly}(\log d)$ if $\mathcal{C}$ is contained in the unit $\ell_2$-norm ball. To the best of our knowledge, this is the first lower bound of DP-SCO that depends on Gaussian width.
 
 \noindent {\bf 2.} We then study problem (2). Specifically, when $1<p<2$, we propose a novel method named Noisy Regularized Mirror Descent, which adds  regularization terms and Generalized Gaussian noise to Mirror Descent. By analyzing its stability, we show the output could achieve an excess population risk of $\tilde{O}(\kappa^\frac{4}{5}(\frac{\sqrt{d\log(1/\delta)}}{n\epsilon})^{\frac{2}{5}})$, where $\kappa = \min \{\frac{1}{p-1},2\log d\}$. We also discuss the case when $2\leq p\leq \infty$. 
 
 \noindent {\bf 3.} Finally, we consider problem (3), assuming that the second-order moment of $\|\cdot\|_*$-norm of the loss gradient is bounded. When $1<p<2$, through a noisy, shuffled, and truncated version of Mirror Descent, we show a bound of $\tilde{O}(
\frac{\sqrt[4]{\kappa^2 d\log(1/\delta)})}{\sqrt{n\epsilon}}
)$ in the high privacy regime $\epsilon=\tilde{O}(n^{-\frac{1}{2}})$, and a bound of $O(\frac{\kappa^\frac{2}{3}(d\log (1/\delta))^\frac{1}{6} }{(n\epsilon)^\frac{1}{3}})$ for general $0<\epsilon<1$. We also study the case when $2\leq p\leq \infty$.

\section{Related Work}
As there is a long list of work on DP-SCO/DP-ERM, here we just mention the work close to the problems we study in this paper. See Table \ref{tab:1} and \ref{tab:2} for detailed comparisons. 
 
\noindent {\bf DP-SCO/DP-ERM with Gaussian width.} For DP-ERM over $\ell_2$-norm ball, although \cite{bassily2014private} show the optimal rate of $O(\frac{\sqrt{d\log(1/\delta)}}{n\epsilon})$ and  $O(\frac{d\log (1/\delta)}{n^2\epsilon^2})$ for convex and strongly convex loss, respectively, \cite{talwar2014private} show that for general constraint set $\mathcal{C}$ it is possible to improve the factor $d$ to the Gaussian width of $\mathcal{C}$. After that, \cite{kasiviswanathan2016efficient} further improve the rate for generalized linear functions, \cite{wang2017differentially} provide an accelerated algorithm, and \cite{wang2019differentially} extend to non-convex loss functions. However, all of them only study the problem of DP-ERM, and their methods cannot be generalized to DP-SCO directly. For DP-SCO, the only known result is given by \cite{amid2022public}, which studies general convex loss under the setting where there is some public data. As we can see from Table \ref{tab:1}, our result significantly improves theirs. Moreover, we  show a nearly optimal rate for strongly convex functions, which is the first lower bound of DP-SCO/DP-ERM that depends on the Gaussian width. 

\noindent {\bf DP-SCO in $\ell_p^d$ space.} Compared to the Euclidean space case, there is little work on DP-SCO in non-Euclidean ($\ell_p^d$) space. \cite{bassily2021non} provide the first study of the problem for $1\leq p\leq \infty$ and propose several results for $p=1$, $1<p<2$ and $2\leq p\leq \infty$. Later \cite{han2022private} further extend to the online setting. However, all the previous algorithms and utility analyses highly rely on the assumption that the diameter of the constrained set is bounded and known, i.e., their results will not hold in the unconstrained case, which is more difficult than the constrained case. In this paper, we fill the gap by providing the first results for unconstrained DP-SCO in $\ell_p^d$ space by proposing several new methods.

\begin{table*}[t]
\begin{center}
\resizebox{\textwidth}{!}{%
\begin{tabular}{|l|l|l|l|l|}
\hline
Methods                                                                                                          & Problem          & Assumption                       & Convex Bound               & Strongly Convex Bound                  \\ [1ex] \hline

\cite{talwar2014private} &  ERM  & Lipschitz & $\tilde{O}(\frac{{G_{\mathcal{C}}}}{n\epsilon})$ & $\tilde{O}(\frac{{G^2_{\mathcal{C}}}}{n^2\epsilon^2})$  \\[1ex]
					\hline
\cite{kasiviswanathan2016efficient}   & ERM                                                                   & Lipschitz and GLM      &      $\tilde{O}(\frac{\sqrt{G_\mathcal{C}}}{\sqrt{n\epsilon}})$       &    ---          \\ [1ex]\hline
\cite{amid2022public}    &  SCO                                                                                                                     & Lipschitz      &        $\tilde{O}(\frac{\sqrt{G_\mathcal{C}}}{\sqrt{n}n^{1/4}_{public}}+\frac{1}{\sqrt{n}})$        & ---             \\ [1ex]\hline
      {\bf This paper}           & SCO   & Lipschitz       & $\tilde{O}(\frac{{G_{\mathcal{C}}}}{n\epsilon}+\frac{1}{\sqrt{n}})$ & $\tilde{O}(\frac{{G^2_{\mathcal{C}}}}{n^2\epsilon^2}+\frac{1}{n})$  (*)             \\[1ex] \hline
\end{tabular}}
\caption{Comparisons on the results for $(\epsilon,\delta)$ DP-SCO/DP-ERM in Euclidean space with bounded constraint set $\mathcal{C}$ (dependence on other parameters are omitted). Here $G_\mathcal{C}$ is the Gaussian width of $\mathcal{C}$, $n$ is the sample size, and $n_{public}$ is the size of public data. $\tilde{O}$ hides other logarithmic factors. (*): We also show such a bound is nearly optimal when $\mathcal{C}$ is contained in unit $\ell_2$ ball. }
\label{tab:1}
\end{center}
\end{table*}

\begin{table*}[!th]
\begin{center}
\resizebox{\textwidth}{!}{%
\begin{tabular}{|l|l|l|l|l|}
\hline
Methods                                                                                                          & Constrained         & Assumption                      & Bound for $\ell^d_p$ ($1<p<2$)              &  Bound for $\ell^d_p$ ($2\leq p\leq \infty$)                \\ [1ex] \hline

\cite{bassily2021non} &  Yes  & Lipschitz & $\tilde{O}(\sqrt{\frac{\kappa}{n}}+\frac{\kappa\sqrt{d}}{n\epsilon})$ & $\tilde{O}(\frac{d^{\frac{1}{2}-\frac{1}{p}}}{\sqrt{n}}+\frac{d^{1-\frac{1}{p}}}{n\epsilon})$  \\[1ex]
					\hline
      {\bf This paper}          & No   & Lipschitz       & $\tilde{O}(\kappa^\frac{4}{5}\cdot (\frac{\sqrt{d}}{n\epsilon})^{\frac{2}{5}}   )$ & $\tilde{O}(d^{1-\frac{2}{p}}(\frac{1}{\sqrt{n}}+\frac{\sqrt{d}}{\epsilon n})) $ \\[1ex] \hline
            {\bf This paper}           & Yes   & Heavy-tailed      & $\tilde{O}(
\frac{\sqrt[4]{\kappa^2 d})}{\sqrt{n\epsilon}}
)$/$\tilde{O}(\frac{\kappa^\frac{2}{3}(d)^\frac{1}{6} }{(n\epsilon)^\frac{1}{3}})$ (*)& $\tilde{O}(\frac{d^{\frac{3}{2}-\frac{1}{p}} }{\sqrt{n}}+\frac{d^{\frac{3}{2}-\frac{1}{2p}} }{\sqrt{n\epsilon}})$               \\[1ex] \hline
\end{tabular}}
\caption{Comparisons on the results for $(\epsilon,\delta)$ DP-SCO in $\ell^d_p$ space with $1<p\leq \infty$ (dependence on other parameters are omitted). Here $d$ is the dimension, $n$ is the sample size, and $\kappa=\min\{\frac{1}{p-1}, 2\log d\}$. $\tilde{O}$ hides other logarithmic factors. (*): The first bound is for the case of $\epsilon=\tilde{O}(n^{-\frac{1}{2}})$ and the second one is for general $0<\epsilon<1$.} 
\label{tab:2}
\end{center}
\end{table*}

\section{Preliminaries}
In this section, we recall some definitions and lemmas that would be used throughout the paper.
\begin{definition}[Differential Privacy \cite{dwork2006calibrating}]\label{def:1}
	Given a data universe $\mathcal{X}$, we say that two datasets $D, D'\subseteq \mathcal{X}$ are neighbors if they differ by only one data sample, which is denoted as $D \sim D'$. A randomized algorithm $\mathcal{A}$ is $(\epsilon,\delta)$-differentially private (DP) if for all neighboring datasets $D,D'$ and for all events $S$ in the output space of $\mathcal{A}$, we have $\text{Pr}(\mathcal{A}(D)\in S)\leq e^{\epsilon} \text{Pr}(\mathcal{A}(D')\in S)+\delta.$
\end{definition}
\begin{lemma}[Advanced Composition Theorem \cite{dwork2014algorithmic}]\label{lemma:adv}
Given target privacy parameters $0<\epsilon <1$ and $0<\delta<1$, to ensure $(\epsilon, T\delta'+\delta)$-DP over $T$ mechanisms, it suffices that each mechanism is $(\epsilon',\delta')$-DP, where $\epsilon'=\frac{\epsilon}{2\sqrt{2T\ln(2/\delta)}}$ and $\delta'=\frac{\delta}{T}$.  
\end{lemma}
	\begin{definition}[DP-SCO in General Normed Space \cite{bassily2021non}]\label{def:sco}
		Given a dataset $D=\{x_1,\cdots,x_n\}$ from a data universe $\mathcal{X}$ where  $\{x_i=(z_i, y_i)\}_{i}$ with a feature vector $z_i$ and a label/response $y_i$ are i.i.d. samples from some unknown distribution $\mathcal{D}$, a normed space $(\mathbf{E}, \|\cdot\|)$ of dimension $d$, 
  a convex constraint set  $\mathcal{C} \subseteq \mathbf{E}$, and a convex loss function $\ell: \mathcal{C}\times \mathcal{X}\mapsto \mathbb{R}$. Differentially Private Stochastic Convex Optimization (DP-SCO) is to find $\theta^{\text{priv}}$ to minimize the population risk, {\em i.e.,} $\mathcal{L} (\theta)=\mathbb{E}_{x\sim \mathcal{D}}[\ell(\theta, x)]$
		with the guarantee of being differentially private.\footnote{Note that in this paper we consider the proper learning case, that is $\theta^{\text{priv}}$ should be in $\mathcal{C}$.} 
		 The utility of the algorithm is measured by the (expected) excess population risk, that is  $\mathcal{L} (\theta^{\text{priv}})-\mathcal{L }(\theta^*),$ where  $\theta^*=\arg\min_{\theta \in \mathbb{\mathcal{C}}}\mathcal{L}(\theta).$	 Besides the population risk, we can also measure the \textit{empirical risk} of dataset $D$: $\hat{\mathcal{L}}(\theta, D)=\frac{1}{n}\sum_{i=1}^n \ell(\theta, x_i).$  
	\end{definition}
In Definition \ref{def:sco}, we consider DP-SCO in general normed space with a convex set $\mathcal{C}\subseteq \mathbf{E}$. In this paper, we mainly focus on two cases: 1) Constraint Euclidean case  where $\mathbf{E}=\mathbb{R}^d$, $\|\cdot \|$ is the $\ell_2$-norm, and $\mathcal{C}$ is a bounded set whose diameter is denoted as $\|\mathcal{C}\|_2=\max_{\theta, \theta'\in \mathcal{C}}\|\theta-\theta'\|_2$; 2)
$\ell_q^d$ case where $\mathbf{E}=\mathbb{R}^d$ and $\|\cdot\|$ is the $\ell_p$-norm $\|\cdot\|_p$ with  $1<p\leq \infty$ (where $||x||_p = (\sum_{j=1}^d |x_j|^p)^{\frac{1}{p}}$), and $\mathcal{C}$ could be either bounded or unbounded. 
Since $\ell^d_p$ spaces are regular. To better illustrate our idea,  we will introduce regular spaces.

Let $(\mathbf{E},||\cdot||)$ be a normed space of dimension $d$ and let $\langle \cdot,\cdot\rangle$ be an arbitrary inner product over $\mathbf{E}$ (not necessarily inducing the norm $\|\cdot\|$). The dual norm over $\mathbf{E}$ is defined as $||y||_{*}=\underset{||x\|\leq 1}{\max}\langle y,x\rangle$. So $(\mathbf{E}, ||\cdot||_{*})$ is also a $d$-dimensional normed space. For example, let  $\ell_p^d = (\mathbb{R}^d,||\cdot||_p)$ with $1\leq p\leq \infty$, the dual norm of $\ell_p^d$ is $\ell_q^d$, where $\frac{1}{p}+\frac{1}{q}=1$.

We call a normed space regular if its dual norm is sufficiently smooth. In detail, we have the following definition. 


\begin{definition}[$\kappa$-regular Space \cite{juditsky2008large}]
Given $\kappa\geq 1$, we say a normed space $(\mathbf{E},||\cdot||)$  $\kappa$-regular if there exists a $\kappa_{+}$, s.t., $ 1\leq \kappa_{+}\leq \kappa$ and there exists a norm $||\cdot||_{+}$ such that $(\mathbf{E},||\cdot||_{+})$ is $\kappa_{+}$-smooth, i.e., for all $x,y\in\mathbf{E}$, 
\begin{equation*}
    ||x+y||_{+}^2\leq ||x||_{+}^2+\langle \nabla (||\cdot||_{+}^2)(x),y\rangle +\kappa_{+}||y||_{+}^2.
\end{equation*}
And $||\cdot||$ and $||\cdot||_{+}$ are equivalent with the following constraint:
    $||x||^2\leq ||x||_{+}^2 \leq \frac{\kappa}{\kappa_{+}}||x||^2 ~(\forall x\in \mathbf{E}).$
\end{definition}
For  $\ell_p^d$ space with $2\leq p\leq \infty$, it is $\kappa$-regular with $\kappa = \min\{p-1,2e\log d\}$. In this case we have $\|x\|_{+}=\|x\|_r$ with $r=\min\{p, 2\log d+1\}$ and $\kappa_+=(r-1)$ \cite{dumbgen2010nemirovski}. So in the $\ell_p$ spaces with $1<p<2$ we focus on, their dual spaces are $\kappa$-regular with $\kappa = \min\{\frac{1}{p-1}, 2\ln d\}$.

In the following, we introduce the mechanisms that will be used in the latter sections. 

\begin{lemma}[Gaussian Mechanism]\label{le-gaussian}
	Given a dataset $D\in\mathcal{X}^n$ and a function $q : \mathcal{X}^n\rightarrow \mathbb{R}^d$, the Gaussian mechanism is defined as  $q(D)+\xi$ where $\xi\sim \mathcal{N}(0,\frac{2\Delta^2_2(q)\log(1.25/\delta)}{\epsilon^2}\mathbb{I}_d)$,  where $\Delta_2(q)$ is the $\ell_2$-sensitivity of the function $q$,
{\em i.e.,}
		$\Delta_2(q)=\sup_{D\sim D'}\|q(D)-q(D')\|_2.$	Gaussian mechanism preserves $(\epsilon, \delta)$-DP.
\end{lemma}

Note that the Gaussian mechanism is tailored for the case where the query has bounded $\ell_2$-norm sensitivity. \cite{bassily2021non} propose a Generalized  Gaussian mechanism that leverages the regularity of the dual space $(\mathbf{E}, \|\cdot\|_*)$. 

\begin{definition}[Generalized Gaussian distribution \cite{bassily2021non}]
\label{def2}
Let $(\mathbf{E},||\cdot||_{*})$ be a $d$-dimensional $\kappa$-regular space with smooth norm $||\cdot||_{+}$. Define the generalized Gaussian distribution $\mathcal{GG}_{||\cdot||_{+}}(\mu,\sigma^2)$, as one with density $g(z) = C(\sigma,d)\cdot e^{-\frac{||z-\mu||_{+}^2}{2\sigma^2}}$, where $C(\sigma,d) = [\text{Area}(\{||x||_{+}=1\})\frac{(2\sigma^2)^{d/2}}{2}\Gamma(\frac{d}{2})]^{-1}$, and the Area is the $d-1$ dimensional surface measure on $\mathbb{R}^d$.
\end{definition}
\begin{lemma}[Generalized Gaussian mechanism \cite{bassily2021non}]
    Given a dataset $D\in\mathcal{X}^n$, and a query $q: \mathcal{X}^n\rightarrow \mathbf{E}$ with bounded $||\cdot||_{*}$-sensitivity: $s= {\sup}_{D\sim D^{'}}||q(D)-q(D^{'})||_{*}$, the Generalized Gaussian mechanism is defined as  $q(D)+\xi$ where $\xi\sim \mathcal{GG}_{||\cdot||_{+}}(0, \frac{2\kappa \log(1/\delta)s^2}{\epsilon^2})$. 	The Generalized Gaussian mechanism preserves $(\epsilon, \delta)$-DP.
\end{lemma}
\begin{lemma}[Prop 4.2 in \cite{bassily2021non}]
    For any $m\geq 1$, if $z\sim \mathcal{GG}_{||\cdot||_{+}}(0,\sigma^2)$, then $\mathbb{E}[\|z\|_+^m]\leq (2\sigma^2)^\frac{m}{2}\Gamma(\frac{m+d}{2})/\Gamma(\frac{d}{2})$. Specifically, $\mathbb{E}[\|z\|_*^2]\leq \mathbb{E}[\|z\|_+^2]\leq d\sigma^2$, where $\Gamma(\cdot)$ is the Gamma function. 
\end{lemma}
In the following, we recall some terminologies on the properties of the loss function and the constraint set $\mathcal{C}$. 
\begin{definition}
{($L$-Lipschitz)}
Given the loss function $\ell(\cdot, \cdot):\mathcal{C} \times \mathcal{X}\rightarrow \mathbb{R}$. It is $L$-Lipschitz w.r.t. the norm $||\cdot||$ if for all $x\in \mathcal{X}$ and $w_1,w_2\in \mathcal{C}$ we have 
\begin{equation}
    |\ell(w_1,x)-\ell(w_2,x)|\leq L\cdot ||w_1-w_2||. \notag 
\end{equation}
\end{definition}
\begin{definition}
{($\beta$-Smooth)}
Given the loss function $\ell(\cdot, \cdot):\mathcal{C} \times \mathcal{X}\rightarrow \mathbb{R}$. It is $\beta$-smooth w.r.t. the norm $||\cdot||$ if its gradient is $\beta$-Lipschitz w.r.t. $||\cdot||$, namely, 
for all $x\in \mathcal{X}$ and $w_1,w_2\in \mathcal{C}$ we have 
\begin{equation}
    ||\nabla\ell(w_1,x)-\nabla\ell(w_2,x)||_{*}\leq \beta\cdot ||w_1-w_2||. \notag 
\end{equation}
\end{definition}
\begin{definition}{(Strongly convex)}
Given the loss function $\ell(\cdot, \cdot):\mathcal{C} \times \mathcal{X}\rightarrow \mathbb{R}$, it is $\alpha$-strongly convex w.r.t. the norm $||\cdot||$ if for all $x\in \mathcal{X}$ and $w_1, w_2\in \mathcal{C}$,
\begin{equation*}
    \langle \nabla \ell(w_1, x)-\nabla \ell(w_2, x), w_1-w_2\rangle \geq \alpha \cdot||w_1-w_2||^2.
\end{equation*}
\end{definition}
\begin{definition}{(Bregman divergence)}
For a convex function $\Phi:\mathbf{E}\rightarrow \mathbb{R}$, the Bregman divergence is defined as 
\begin{equation}
    D_{\Phi}(y,x) = \Phi(y)-\Phi(x)-\langle \nabla \Phi(x),y-x\rangle. \notag 
\end{equation}
\end{definition}
Notice that the Bregman divergence is always positive, and it is convex in the first argument.
\begin{definition}{(Relative strongly convex \cite{lu2018relatively})}
A function $f:\mathbf{E}\rightarrow \mathbb{R}$ is $\alpha$-strongly convex \textbf{relative}  to $\Phi:\mathbf{E}\rightarrow \mathbb{R}$ if for all $x, y\in \mathbf{E}$,
\begin{equation}
    f(x)+\langle \nabla f(x),y-x\rangle +\alpha D_{\Phi}(y,x)\leq f(y). \notag 
\end{equation}
\end{definition}

\begin{definition}{(Relative smooth \cite{lu2018relatively})}
A function $f:\mathbf{E}\rightarrow \mathbb{R}$ is $\beta$-smooth \textbf{relative}  to $\Phi:\mathbf{E}\rightarrow \mathbb{R}$ if $\forall x, y\in \mathbf{E}$,
    $f(x)+\langle \nabla f(x),y-x\rangle +\beta D_{\Phi}(y,x)\geq  f(y).$
\end{definition}
Next, we introduce some basic concepts on Minkowski norm of a symmetric, closed, and convex set $\mathcal{C}$.

\begin{definition}[Minkowski norm]
For a centrally symmetric convex set $\mathcal{C}\subseteq \mathbb{R}^d$, the Minkowski norm (denoted by  $||\cdot||_{\mathcal{C}}$) is defined as follows. For any vector $v\in \mathbb{R}^d$, 
\begin{equation*}
    ||\cdot||_{\mathcal{C}}=\min\{r\in \mathbb{R}^{+}: v\in r \mathcal{C}\}.
\end{equation*}
The dual norm of $||\cdot||_{\mathcal{C}}$ is denoted as $||\cdot||_{\mathcal{C}^{*}}$, and for any vector $v\in \mathbb{R}^d$, 
    $||v||_{\mathcal{C}^{*}} =\underset{w\in \mathcal{C}}{\max}|\langle w, v\rangle|$.
Note that by Holder's inequality, for any pair of dual norms $||\cdot||$ and $||\cdot||_{*}$, and any $x,y\in \mathbb{R}^d$, $|\langle x, y\rangle |\leq ||x||\cdot||y||_{*}$. So we have $|\langle x,y\rangle |\leq ||x||_{\mathcal{C}}\cdot||y||_{\mathcal{C}^{*}}$.
\end{definition}
In the constrained Euclidean case, 
our work relies on appropriately quantifying the size of  a convex body, which leads to the following definition of Gaussian width.

\begin{definition}\label{def:12}(Gaussian width)
Let $\xi\sim \mathcal{N}(0,\mathbb{I}_d)$ be a Gaussian random vector in $\mathbb{R}^d$, for a set $\mathcal{C}$, the Gaussian width is defined as 
    $G_{\mathcal{C}}=\mathbb{E}_{\xi}[\underset{w\in\mathcal{C}}{\sup}\langle \xi,w \rangle ].$

\end{definition}
Compared to dimension $d$, the Gaussian width
of a convex set $\mathcal{C}\subset \mathbb{R}^d$ could be much smaller. For example, when $\mathcal{C}$ is the unit $\ell_1$-norm  ball, $G_\mathcal{C}=O(\sqrt{\log d})$; and when $\mathcal{C}$ is the set of  of all unit s-sparse vectors on $\mathbb{R}^d$, $G_\mathcal{C}=O(\sqrt{s\log \frac{d}{s}})$. We refer readers to  \cite{talwar2014private} for details. 
\section{DP-SCO in Euclidean Space}\label{sec:euclidean}
In this section, we focus on the Euclidean case with a closed, bounded, and convex constraint set $\mathcal{C}$, and the loss function could be either convex or strongly convex.  
\subsection{General Convex Case}\label{sec:convex}
Before showing our idea, we need to discuss the weakness of previous approaches. Note that since our goal is getting an upper bound that depends on the Gaussian width of the constrained set $\mathcal{C}$, we will not discuss the approaches that achieve upper bounds that are polynomial in $d$. 

In general, all methods can be categorized into two classes: gradient perturbation and objective function perturbation. In gradient perturbation methods \cite{talwar2014private}, the key idea is modifying the Mirror Descent by adding noise to gradients. While this approach could achieve satisfactory bounds for the empirical risk  \cite{wang2017differentially,wang2019differentially}, however, when considering the population risk we need to use batched gradients at each iteration, which will induce a sub-optimal rate \cite{amid2022public}. Instead of perturbing the gradient, \cite{talwar2014private} show that the objective function perturbation method in \cite{chaudhuri2011differentially} could also achieve an upper bound that only depends on the Gaussian width, instead of $d$. However, this approach has two weaknesses: First, \cite{talwar2014private} only shows the bound for the empirical risk, and whether its excess population risk is satisfactory or not is unknown; Secondly, it is well-known that the objective perturbation approach needs to exactly get the minimizer of the perturbed objective function, which is inefficient in practice. 

Motivated by the objective perturbation method in \cite{talwar2014private}, our algorithm is an approximate version proposed in \cite{bassily2019private}. See detailed procedures in Algorithm \ref{alg1}. In detail, first, similar to the standard objective perturbation, we add a random and linear term $\frac{\langle \mathbf{G},\theta\rangle }{n}$ with Gaussian noise  $\mathbf{G}$ and an $\ell_2$ regularization to the original empirical risk function to obtain a new objective function $\mathcal{J}(\theta, D)$. Then we obtain an approximate minimizer $\theta_2$ of the perturbed empirical risk $\mathcal{J}(\theta, D)$ via any efficient optimization method (such as proximal SVRG \cite{xiao2014proximal} or projected SGD) to ensure that $\mathcal{J}(\theta_2,D)-\underset{\theta\in\mathcal{C}}{\min} \mathcal{J}(\theta,D)$ is at most $\alpha$. Formally, we can define such an optimization method  as an optimizer function $\mathcal{O}: \mathcal{F}\times [0, 1]\rightarrow  \mathcal{C}$, where $\mathcal{F}$ is the class of objectives and the other argument is the optimization accuracy.  Finally, we perturb $\theta_2$ with Gaussian noise to fuzz the difference 
between $\theta_2$ and the true minimizer, we then project the perturbed $\theta_2$ onto  set $\mathcal{C}$.

Since the algorithm itself is not new, here we  highlight our contributions: First, with some specific parameters, we show such an algorithm could achieve an excess population risk of $O(\frac{G_\mathcal{C}}{n\epsilon}+\frac{1}{\sqrt{n}})$, while \cite{bassily2019private} only show an upper bound of $O(\frac{\sqrt{d}}{n\epsilon}+\frac{1}{\sqrt{n}})$; Second, we extend the algorithm to the strongly convex case (see Section \ref{sec:strongconv} for details). In the following, we will show the theoretical guarantees of our algorithm. First, we need the following assumption on the loss function $\ell$.


\begin{assumption}\label{ass1}
The loss function $\ell$ is  twice differentiable, $L$-Lipschitz and $\beta$-smooth w.r.t. the Euclidean norm $||\cdot||_2$ over $\mathcal{C}$. 
\end{assumption}
\begin{algorithm}
	\caption{$\mathcal{A}_{\text{App-ObjP}}$: Approximate Objective perturbation }
	\begin{algorithmic}[1]
		\State {\bfseries Input:} Dateset $D$, loss function $\ell$, regularization parameter $\lambda$, optimizer $\mathcal{O}: \mathcal{F}\times [0, 1]\rightarrow  \mathcal{C}$, where $\mathcal{F}$ is the class of objectives, and the other argument is the optimization accuracy. $\alpha \in [0,1]:$ optimization accuracy.

	\State  Sample $\mathbf{G}\sim\mathcal{N}(0, \sigma_1^2\mathbb{I}_d)$ where $\sigma_1^2  = \frac{128L^2\log(2.5/\delta)}{\epsilon^2}$. Set $\lambda\geq \frac{r\beta}{\epsilon n}$, where $r = \min\{d, 2\cdot \text{rank}(\nabla^2\ell(\theta, x))\}$ with $\text{rank}(\nabla^2\ell(\theta,x))$ being the maximal rank of the Hessian of $\ell$ for all $\theta \in \mathcal{C}$ and $x\sim \mathcal{P}$.
	
	\State Let $\mathcal{J}(\theta,D) = \hat{\mathcal{L}}(\theta, D)+\frac{\langle \mathbf{G},\theta\rangle}{n}+\lambda||\theta||_2^2$.\\
	 \Return $\hat{\theta} = {\text{Proj}}_{\mathcal{C}}[\mathcal{O}(\mathcal{J},\alpha)+\mathbf{H}]$ where $\mathbf{H}\sim\mathcal{N}
	 (0,\sigma_2^2\mathbb{I}_d)$ 
	 with $\sigma_2^2 = 
	 \frac{64\alpha
	 \log(2.5/\delta)}{\lambda\epsilon^2}$
	\end{algorithmic}
	\label{alg1}
\end{algorithm}
\begin{theorem}\label{thm:1}
Suppose that Assumption \ref{ass1} holds and that the smoothness parameter $\beta$ satisfies $\beta \leq \frac{\epsilon n \lambda}{r}$. Then for any $0< \epsilon, \delta< 1$, 
$\mathcal{A}_{\text{App-ObjP}}$ (Algorithm \ref{alg1}) is $(\epsilon, \delta)$-DP.
\end{theorem}
It is notable that although we need to assume $\beta$ is not large enough, as we will see in Theorem \ref{th4}, the assumption will always hold when $n$ is sufficiently large. 
\begin{theorem}\label{th4}
Suppose that Assumption \ref{ass1} holds. 
When $n$ is large enough such that  $n\geq \frac{r^2\beta^2||\mathcal{C}||_2^2}{\epsilon^2L^2}$ and $n\geq O\left(\frac{\sqrt{d\log(1/\delta)}}{\epsilon}\right)$, take $\lambda=\frac{L}{\sqrt{n}||\mathcal{C}||_2}$ and $\alpha\leq \min\left\{
\frac{L||\mathcal{C}||_2}{n^{\frac{3}{2}}}, \frac{\epsilon^2 L||\mathcal{C}||_2^3}{G^2_{\mathcal{C}}\log(1/\delta) n^{\frac{5}{2}}}
\right\}$ in Algorithm \ref{alg1}, we have 
\begin{equation*}
    \mathbb{E}[\mathcal{L}(\hat{\theta})] - \mathcal{L}(\theta^{*})\leq O\left(\frac{L\cdot G_{\mathcal{C}}\sqrt{\log(1/\delta)}}{\epsilon n}+\frac{L||\mathcal{C}||_2}{\sqrt{n}}\right),
\end{equation*}
where  the expectation is taken over the internal randomness of the algorithm.
\end{theorem}
\begin{remark}
While we consider the same algorithm as in \cite{bassily2019private}, there are several crucial differences. First, to achieve the upper bound of $O(\frac{\sqrt{d}}{n\epsilon}+\frac{1}{\sqrt{n}})$, \cite{bassily2019private} only need to set $\alpha\leq O(\frac{1}{n^2}\max\{\frac{1}{\sqrt{n}}, \frac{d}{n\epsilon}\})$ while we need to be more aggressive by choosing $\alpha\leq O({\epsilon^2}{n^{-\frac{5}{2}}})$. This is reasonable as we aim to get an improved upper bound. Thus we have to get a more accurate estimation. Secondly, besides enforcing the perturbed approximation to lie in the set $\mathcal{C}$ as it does in \cite{bassily2019private}, the projection operator in Step 4 of Algorithm \ref{alg1} plays a more critical role in achieving a bound that depends on $G_\mathcal{C}$ in our analysis, i.e., the bound in \cite{bassily2019private} will still hold even there is no projection step, while this is not true for our case. Specifically, although the noise $\mathbf{H}$ is a $d$-dimensional Gaussian noise, we can show that due to the projection operator, the error introduced by the noise  depends only on $G_\mathcal{C}$ rather than $\sqrt{d}$, i.e., $||\hat{\theta}-\theta_2||_2^2\leq O(\sqrt{\frac{\alpha \log(1/\delta)}{\lambda}}\cdot \frac{G_{\mathcal{C}}}{\epsilon})$. A similar idea has also been used in privately answering multiple linear queries \cite{nikolov2016geometry}. 
\end{remark}
\subsection{Strongly Convex Case}\label{sec:strongconv}
We aim to extend our above idea to the strongly convex case. First, we impose 
the following assumption. 
\begin{assumption}\label{ass2}
We assume the loss is twice differentiable, $L$-Lipschitz and $\beta$-smooth w.r.t. $||\cdot||_2$, and it is $\Delta$-strongly convex w.r.t. $||\cdot||_{\mathcal{C}}$ over the set $\mathcal{C}$.
\end{assumption}
Note that we can relax the assumption to strongly convex w.r.t $\|\cdot\|_2$ as $\|v\|_2\geq \mathcal{C}_{\min}\cdot \|v\|_\mathcal{C}$, where $\mathcal{C}_{\min}$ is in Theorem \ref{th1}. See the proof of Theorem \ref{th1} for details.

Our method is shown in  Algorithm \ref{alg2}. Note that, compared with Algorithm \ref{alg1}, the main difference is the regularization parameter $\lambda$. This is because the loss function is already $\Delta$-strongly convex, thus smaller $\lambda$ will be sufficient to make $\mathcal{J}$ to be $\frac{r\beta}{\epsilon n}$-strongly convex. Moreover, when $n$ is large enough, we can see $\lambda=0$, indicating that we can get an improved excess population risk  compared to the convex case. 
\begin{algorithm}
	\caption{$\mathcal{A}_{\text{App-ObjP-SC}}$: Approximate Objective perturbation for strongly convex function }
	\begin{algorithmic}[1]
		\State {\bfseries Input:} Dateset $D$, loss function $\ell$, regularization parameter $\lambda$, optimizer $\mathcal{O}: \mathcal{F}\times [0, 1]\rightarrow  \mathcal{C}$, where $\mathcal{F}$ is the class of objectives and the other argument is the optimization accuracy. $\alpha \in [0,1]:$ optimization accuracy.

	\State  Sample $\mathbf{G}\sim\mathcal{N}(0, \sigma_1^2\mathbb{I}_d)$ where $\sigma_1^2  = \frac{128L^2\log(2.5/\delta)}{\epsilon^2}$. Set $\lambda= \max\left\{
	\frac{r \beta}{\epsilon n}-\Delta,0
	\right\}$, where $r = \min\{d, 2\cdot \text{rank}(\nabla^2\ell(\theta, x))\}$ with $\text{rank}(\nabla^2\ell(\theta,x))$ being the maximal rank of the Hessian of $\ell$ for all $\theta \in \mathcal{C}$ and $x\sim \mathcal{P}$.
	
	\State Let $\mathcal{J}(\theta,D) = \hat{\mathcal{L}}(\theta, D)+\frac{\langle \mathbf{G},\theta\rangle}{n}+\lambda||\theta||_2^2$.\\
	 \Return $\hat{\theta} = {\text{Proj}}_{\mathcal{C}}[\mathcal{O}(\mathcal{J},\alpha)+\mathbf{H}]$ where $\mathbf{H}\sim\mathcal{N}
	 (0,\sigma_2^2\mathbb{I}_d)$ 
	 with $\sigma_2^2 = 
	 \frac{64\alpha
	 \log(2.5/\delta)\cdot ||\mathcal{C}||_2^2}{\Delta\epsilon^2}$
	\end{algorithmic}
	\label{alg2}
\end{algorithm}
\begin{theorem}\label{thm:3}
If the loss function satisfies Assumption \ref{ass2}. Then for any $0<\epsilon, \delta<1$, 
$\mathcal{A}_{\text{App-ObjP-SC}}$ (Algorithm \ref{alg2}) is $(\epsilon,\delta)$-DP.
\end{theorem}
\begin{theorem}\label{th5} Suppose that Assumption \ref{ass2} holds. If $n$ is large enough such that 
 $n\geq O(\max\{\frac{L^2 ||\mathcal{C}||_2^2}{\Delta^2}, \frac{||\mathcal{C}||_2^2 r^2\beta^2}{L^2\epsilon^2}
\})$ and $n\geq O\left(\frac{\sqrt{d\log(1/\delta)}}{\epsilon}\right)$, then  by setting $\alpha\leq O\left(\min\left\{
\frac{L^2 ||\mathcal{C}||_2^2}{\Delta n^2},\frac{L^4\cdot ||\mathcal{C}||_2^6\epsilon^2}{\Delta^3 n^4 G_{\mathcal{C}}^2\log(1/\delta)}
\right\}
\right)$, we have 
\begin{equation}
    \mathbb{E}[\mathcal{L}(\hat{\theta})]-\mathcal{L}(\theta^{*})\leq O\left(
    \frac{L^2 ||\mathcal{C}||_2^2}{\Delta n\epsilon} + \frac{G_{\mathcal{C}}^2L^2\log(1/\delta)}{\Delta n^2\epsilon^2}
    \right), \notag 
    \end{equation}
    where  the expectation is taken over the internal randomness of the algorithm.
\end{theorem}
\begin{remark}\label{remark:2}
 First, it is notable that an objective perturbation method for strongly convex loss has also been presented by \cite{talwar2014private}. However, there are two major differences: (1) the method in \cite{talwar2014private} needs to solve the perturbed objective function exactly, indicating it is inefficient; (2) \cite{talwar2014private} only provide the excess empirical risk. It is unknown whether their method could achieve the same bound as ours for the excess population risk.  Secondly, when $\mathcal{C}$ is an $\ell_2$-norm ball, the bounds in  Theorem \ref{th4} and Theorem \ref{th5} will recover the optimal rate of DP-SCO over $\ell_2$-norm ball for convex and strongly convex loss functions, respectively \cite{bassily2019private}.   Thirdly, the terms of $O(\frac{G_\mathcal{C}}{n\epsilon})$ and $O(\frac{G^2_\mathcal{C}}{n^2\epsilon^2})$ match the best-known results of excess empirical risk for the convex  and strongly convex case, respectively \cite{talwar2014private}. 
\end{remark}
In Remark \ref{remark:2}, we showed that our results are optimal when $\mathcal{C}$ is an $\ell_2$-norm ball and are comparable to the best results of DP-ERM with Gaussian width. A natural question is whether we can further improve these two upper bounds. In the following, we partially answer the question by providing a lower bound for  strongly convex loss functions. 

\begin{theorem}\label{th1}
Let $\mathcal{C}$ be a symmetric body contained in the unit Euclidean ball $\mathcal{B}_2^d$ in $\mathbb{R}^d$ and satisfies $\|\mathcal{C}\|_2=1$. For any $ n=O(\frac{\sqrt{d\log(1/\delta)}}{\epsilon})$, $\epsilon=O(1)$ and  $2^{-\Omega(n)}\leq \delta\leq 1/n^{1+\Omega(1)}$, there exists a loss $\ell$ which is $1$-Lipschitz w.r.t. $\|\cdot\|_2$ and $\mathcal{C}^2_{\min}$-strongly convex w.r.t. $\|\cdot\|_\mathcal{C}$, and a dataset $D=\{x_1,\cdots,x_n\}\subseteq \mathcal{C}^n$ such as for any $(\epsilon,\delta)$-differentially private algorithm on minimizing the empirical risk function $\hat{\mathcal{L}}(\theta,D)$ over $\mathcal{C}$, its output 
 $\theta^{priv}\in \mathcal{C}$ satisfies 
	$$
	\mathbb{E}[\mathcal{L}(\theta^{priv})]-\mathcal{L}(\theta^*)=\Omega \left(\max\left\{\frac{G_{\mathcal{C}}^2\log(1/\delta)}{(\log(2d))^4\epsilon^2n^2},\frac{1}{n}\right\}\right),$$ where the expectation is taken over the internal randomness of the algorithm $\mathcal{A}$. Here $\mathcal{C}_{\min}= \min\{\|v\|_2: v\in \partial \mathcal{C}\}$ with $\partial \mathcal{C}$ as the boundary of the set $\mathcal{C}$, i.e., it is the distance between the original point to the boundary of $\mathcal{C}$. 
\end{theorem}
    Taking $\Delta=\mathcal{C}^2_{\min}$ and $L=1$ in Theorem \ref{th5}, we can see the rate of excess population risk in Theorem \ref{th5} for strongly convex loss functions is nearly optimal by a factor of $\tilde{O}({\mathcal{C}^{-2}_{\min}})$. It is unknown whether we can further close the gap, and we will leave it as an open problem.  
\section{DP-SCO in $\ell_p^d$ Space} \label{sec:regular}
In this section, we will focus on DP-SCO in $\ell_p^d$ space where $1<p\leq \infty$. As we mentioned in the Introduction section, we study two settings: (1) $\mathcal{C}$ is $\mathbb{R}^d$ and the gradient of the loss function is bounded (i.e., the loss is Lipschitz); (2) $\mathcal{C}$ is bounded, and the distribution of  gradient of the loss is heavy-tailed. Similar to the previous study in \cite{bassily2021non}, for each setting, there are two cases: $1<p< 2$ and $2\leq p\leq \infty$. Notice that, unlike the previous section, we only study the case where the loss functions are convex. The reason is that except  for the Euclidean space, for a strongly convex function, the ratio between its smoothness and strong convexity, i.e., the condition number, will depend on the dimension of $\mathbf{E}$. For example, in the $\ell_1^d$ space, it has been shown that there is no function whose condition number is less than $d$ \cite{juditsky2014deterministic}.
\subsection{Unconstrained Case}
In this part, we will study Lipschitz loss under the following assumption that is commonly used in the related work on general stochastic convex optimization. 
\begin{assumption}\label{as3}
We assume $\ell(\cdot,x)$ is convex, $\beta$-smooth and $L$-Lipschitz w.r.t. $||\cdot||$ over $\mathbb{R}^d$.
\end{assumption}
Due to its difficulty, we first consider the case where $1< p< 2$. See Algorithm \ref{alg3} for details. Note that Algorithm \ref{alg3} could be considered as a noisy and regularized version of the standard mirror descent, i.e., at each iteration, we first perform linearization of $\hat{\mathcal{L}}(w_t, D)$, then we add a generalized Gaussian noise to its gradient to privatize the algorithm, a Bregman divergence term and a regularized term $\alpha\Phi(\cdot)$ with some specific $\alpha$ to the  linearization term. Then we solve the perturbed and regularized optimization problem.  We output a linear combination of the intermediate parameters as the final output. 

It is notable that although our method is a noisy modification of Mirror Descent, it is completely different from the previous private Mirror Descent based methods in \cite{talwar2014private,wang2017differentially,bassily2021non,amid2022public}: First, instead of directly adding noise to the gradient in standard Mirror Descent, here we have an additional regularization term, which is crucial for us to make the algorithm stable, indicating that we can get an excess population risk. To be more specific, first,  by the definition of $\|\cdot\|_{+}$, and  the duality between strong convexity and smoothness, we can easily see $\Phi$  is 1-strongly convex w.r.t $\|\cdot\|$. This indicates that the function $\hat{\mathcal{L}}(w, D)+\alpha \Phi(w)$ is relatively strongly convex and smooth (note that it is not smooth as the regularization term is not smooth when $1<p<2$). And the update step is just a noisy version of Mirror Descent for $\hat{\mathcal{L}}(w, D)+\alpha \Phi(w)$. Recently, it has been shown that Mirror Descent is stable for relatively strongly convex and smooth functions. Thus, we can also show that Algorithm \ref{alg3} is stable, indicating that we can get an excess population risk. From the above intuition, we can also see the parameter $\alpha$ need to be carefully tuned to balance the stability and the excess empirical risk. The second difference is that, instead of using the last iterate or the average of iterates, our output is a linear combination of intermediate iterates, which is due to the noise we added. In the following we show the main results. 

\begin{theorem}\label{thm:6}
For the $\ell_p^d$ space with $1<p<2$, suppose Assumption \ref{as3} holds, then for any $0<\epsilon, \delta<1$, Algorithm \ref{alg3} is $(\epsilon,\delta)$-DP.
\end{theorem}

\begin{theorem}\label{th:7}
For the $\ell_p^d$ space with $1<p<2$, suppose Assumption \ref{as3} holds. In Algorithm \ref{alg3}, take $\alpha=\frac{4\beta}{T}\log_2 \frac{n}{T}$ and $T=O((\frac{n\epsilon\kappa}{\sqrt{d\log(1/\delta)}})^{\frac{2}{5}})$, assume $n$ is sufficiently large such that $n\geq O\left( \frac{\epsilon^4}{(d \log (1/\delta))^2\kappa^{1/2}}\right)$, then we have 
\begin{equation*}
    \mathbb{E}[\mathcal{L}(\hat{w})]-\mathcal{L}(\theta^*)\leq \tilde{O}(\kappa^\frac{4}{5}\cdot (\frac{\sqrt{d\log(1/\delta)}}{n\epsilon})^{\frac{2}{5}}   ),
\end{equation*}
where $\tilde{O}$ hides $\beta, L$ and a factor of $\mathbb{E}_{D}[\tilde{C}_D^2]$ with $\tilde{C}_D^2 = \|\tilde{w}^*\|_{\kappa_+}^2\leq \|\tilde{w}^*\|^2$ and $\tilde{w}^{*}=\underset{w\in\mathbf{E}}{\arg\min} \hat{\mathcal{L}}(w,D)$).
\end{theorem}
 The key idea to prove Theorem \ref{th:7} is to show  that Algorithm \ref{alg3} is uniformly stable (w.r.t $\|\cdot\|$) by bounding the term $\mathbb{E}[\|w_{t+1}-w'_{t+1}\|]$, where $w'_{t+1}$ is the corresponding iterate of the algorithm when the input data is $D'$, which is a neighboring data of $D$. To show this, rather than analyzing the stability of $w_{t+1}$ directly via the approach in \cite{hardt2016train}, our strategy is bounding $\|w_{t+1}-w^*_\alpha\|$, where $w^*_\alpha=\arg\min \hat{\mathcal{L}}(w, D)+\alpha\Phi(w)$. As the regularized function $\hat{\mathcal{L}}(w, D)+\alpha\Phi(w)$ now is relatively smooth and convex, the stability of $w^*_\alpha$ is $O(\frac{1}{n})$. Thus we can get the sensitivity of $w_{t+1}$. Then we can bound the sensitivity of $\hat{w}$. 
\begin{remark}
    In the constrained case, \cite{bassily2021non} show that it is possible to achieve an upper bound of $\tilde{O}((M+M^2)(\frac{\sqrt{\kappa}}{\sqrt{n}}+\frac{\kappa \sqrt{d\log 1/\delta}}{n\epsilon})),$ where  $
    M$ is the diameter of set $\mathcal{C}$. Thus, we can see there is still a gap between the unconstrained case and the constrained case.  
\end{remark}
\begin{algorithm}
	\caption{Noisy Regularized Mirror Descent for $\ell_p^d$ ($1<p<2$).}
	\begin{algorithmic}[1]
		\State {\bfseries Input:} Dateset $D$, loss function $\ell$, smoothness parameter $\beta$ and parameter $\alpha$.
	\State  Take $w_1 = 0$.
	\For{$t = 1,\cdots,T$}
\State Solve the following optimization problem 
\begin{align}
    &w_{t+1}= \underset{w\in\mathbf{E}}{\arg\min}\{
\langle \nabla \hat{\mathcal{L}}(w_t,D)+g_t, w-w_t\rangle \notag \\ &+\beta\cdot D_{\Phi }(w,w_t)+\alpha \Phi(w)
\}, 
\end{align}
 where $g_t\sim \mathcal{GG}_{||\cdot||_{+}}(0,\sigma^2)$ with $\sigma^2 = \frac{64L^2\kappa T\log(1/\delta)}{n^2\epsilon^2}$ and $||\cdot||_{+}$ is the smooth norm for $(\mathbf{E},||\cdot||_{*})$. $\kappa = \min \{\frac{1}{p-1},2\log d\}$ and  $\Phi(x) =\frac{\kappa}{2}||x||_{\kappa_{+}}^2$ with $\kappa_{+}= \frac{\kappa}{\kappa -1}$.  
	\EndFor\\
	 \Return $\hat{w}=\frac{\sum_{t=1}^T \left(\frac{2\beta+\alpha}{2\beta}\right)^t  \cdot w_{t+1}}{\sum_{t=1}^T \left(\frac{2\beta+\alpha}{2\beta}\right)^t }.$
	\end{algorithmic}
	\label{alg3}
\end{algorithm}
Next, we study the case where $2\leq p\leq \infty$. The key idea is to reduce the $\ell_p^d$ space to the Euclidean space by leveraging the relationship between the $\ell_p$ norm and the Euclidean norm. Thus, here we adopt the Phased DP-SGD algorithm proposed by \cite{feldman2020private}. As the parameters in the original Phased DP-SGD depend on the diameter, we modify them to the unconstrained case. Specifically, we have the following result. 
\begin{theorem}\label{thm:8}
For the $\ell_p^d$ space with $2\leq p\leq \infty$, suppose Assumption \ref{as3} holds. Then for any $0<\epsilon, \delta<1$, there is an $(\epsilon,\delta)$-DP algorithm whose output $\theta$ satisfies 
\begin{equation}
   \mathbb{E}[\mathcal{L}(\theta)]-\mathcal{L}(\theta^*)\leq  O(d^{1-\frac{2}{p}}\|\theta^*\|^2(\frac{1}{\sqrt{n}}+\frac{\sqrt{d\log(1/\delta)}}{\epsilon n})). \notag 
\end{equation}
\end{theorem}
In the constrained case, \cite{bassily2021non} shows the optimal rate of $O(M d^{\frac{1}{2}-\frac{1}{p}}(\frac{1}{\sqrt{n}}+\frac{\sqrt{d\log 1/\delta}}{n\epsilon})),$ where $M$ is the diameter of the set $\mathcal{C}$ w.r.t. $\|\cdot\|$. Thus, we can see there is a difference of $O(d^{\frac{1}{2}-\frac{1}{p}})$. This is because, rather than linear in $M$ in the constrained case, in the Euclidean and unconstrained case, we can show the excess population risk depends on $\|\theta^*\|_2^2$, which is less than $d^{1-\frac{2}{p}}\|\theta^*\|^2$. 
\subsection{Heavy-tailed and Constrained Case}
In the above section, we studied DP-SCO with Lipschitz loss functions, i.e., the $\|\cdot\|_*$ norm of  the loss gradient is  uniformly bounded by $L$. Next, we will relax this assumption to a heavy-tailed distribution, i.e., we only assume the  variance of the loss gradient w.r.t $\|\cdot\|_*$ is finite. As we have discussed the difficulty of the unconstrained case compared to the constrained case,  throughout the section, we focus on the constrained case with the $\|\cdot \|$-norm diameter $M$. 
\begin{assumption}\label{ass:4}
We assume $\ell(\cdot,x)$ is convex and $\beta$-smooth  $||\cdot||$ over $\mathcal{C}$.  Moreover, for all  $w\in\mathcal{C}$ there exists a known constant $\sigma>0$ such that 
$\mathbb{E}[||\nabla \ell(w,x)-\nabla \mathcal{L}(w)||_{*}^2]\leq \sigma^2$. 
\end{assumption}
It is noteworthy that the heavy-tailedness assumption is commonly used in previous related work, such as \cite{vural2022mirror}. Besides the norm of gradient, there is another line of work that only assumes the second-order moment of each coordinate of the gradient is bounded \cite{hu2022high,kamath2022improved,wang2020differentially,wang2022differentially,tao2022private}. We leave such a relaxed assumption as future work.

Like the previous section, we first study the case where $1<p<2$.  We present our algorithm in Algorithm \ref{alg4}, which could be considered a shuffled, truncated, and noisy version of one-pass Mirror Descent.  Specifically, in the first step, we shuffle the dataset and divide it into several batches (we will use one batch for one iteration). Using the by-now standard method of privacy amplification by shuffling \cite{feldman2022hiding}, we can amplify the overall privacy guarantee  by a factor of $\tilde{O}(\frac{1}{n})$ as compared to the analysis for the unshuffled dataset. Next, motivated by \cite{nazin2019algorithms}, at each iteration, we first conduct a truncation step to each sample gradient $\nabla \ell(w_{t-1}, x)$. Such an operator can not only remove outliers, but also upper bound the $\|\cdot\|_*$-sensitivity of the truncated gradients to $O(\beta M+\lambda)$. Then we perform the Mirror Descent update by these perturbed and truncated  sample gradients. In the following, we show the privacy and utility guarantees of our algorithm. 
\begin{algorithm}
	\caption{Shuffled Truncated DP Mirror Descent }
	\begin{algorithmic}[1]
		\State {\bfseries Input:} Dataset $D$, loss function $\ell$, initial point $w_0=0$, smooth parameter $\beta$ and   $\lambda$. 
	\State  Randomly permute the data and denote the permuted data as $\{x_1,\cdots,x_n\}$.
 \State Divide the permuted data into $T$ batches $\{B_i\}_{i=1}^T$ where $|B_i|=\frac{n}{T}$ for all $i=1,\cdots, T$
 \For{$t = 1,\cdots,T$}
\For{each $x\in B_t$}
\State \hspace{-3mm}\begin{small}$g_x= \begin{cases} \nabla \ell(w_{t-1},x)& \text{if} ~||\nabla \ell(w_{t-1},x)||_{*}\leq \beta M+\lambda\\ 0& {\text{otherwise}} \end{cases}$\end{small}
	\EndFor
 \State 
Update 
\begin{small} $$w_t=\underset{w\in\mathcal{\mathcal{C}}}{\arg\min} \left\{ \left\langle \frac{\underset{{x\in B_t}}{\sum} g_x + Z_x^t}{|B_t|},w \right\rangle +\gamma_t \cdot D_{\Phi}(w,w_{t-1})\right\},$$ \end{small} where  $Z_x^{t}\sim \mathcal{GG}_{||\cdot||_{+}}(\sigma_1^2)$ with $\sigma_1^2 =O\left(\frac{\log(\frac{n}{\delta})\cdot \kappa (\beta M+\lambda )^2 \cdot \log(1/\delta)}{n\epsilon^2}   \right)$, $||\cdot||_{+}$ is the smooth norm for $(\mathbf{E},||\cdot||_{*})$. $\kappa = \min \{\frac{1}{p-1},2\log d\}$ and  $\Phi(x) =\frac{\kappa}{2}||x||_{\kappa_{+}}^2$ with $\kappa_{+}= \frac{\kappa}{\kappa -1}$.  
	\EndFor\\
	 \Return $\hat{w} = (\sum_{t=1}^T \gamma_{t}^{-1})^{-1} \cdot \sum_{t=1}^T \gamma_{t}^{-1} w_t$
	\end{algorithmic}
	\label{alg4}
\end{algorithm}
\begin{theorem}\label{thm:9}
For the $\ell_p^d$ space with $1<p<2$, suppose Assumption \ref{ass:4} holds. Algorithm \ref{alg4} is $(\epsilon, \delta)$-DP if $\epsilon = O( \sqrt{\frac{\log(n/\delta)}{n}})$ and $0<\delta<1$.
\end{theorem}

\begin{theorem}\label{thm:10}
For the $\ell_p^d$ space with $1<p<2$, suppose Assumption \ref{ass:4} holds and assume $n$ is sufficiently large such that $n\geq O(\frac{\max\{\beta^2, 1\}M^2\sqrt{d\kappa^2 \log(1/\delta)} }{\epsilon})$. Given a failure probability $\delta'>0$, in Algorithm \ref{alg4}, take ${T}=O(\frac{M^2n^2\epsilon^2} {\lambda^2 {d \log (1/\delta)}})$, $\{\gamma\}_{t=1}^T=\bar{\gamma}=\sqrt{T}$, and $\lambda =O(\frac{\sqrt{n\epsilon}}{\sqrt[4]{\kappa^2 d \log (1/\delta)}})$, then the output $\hat{w}$ satisfies the following with probability $1-\delta'$
   \begin{equation*}
\mathbb{E}[\mathcal{L}(\hat{w})]- \mathcal{L}(w^*)\leq  \tilde{O}(
\frac{M\sqrt[4]{\kappa^2 d\log(1/\delta)}\log(1/\delta^{'})}{\sqrt{n\epsilon}}
), 
\end{equation*}
where the expectation is taken over  the randomness of noise, and the probability is w.r.t. the dataset $D\sim \mathcal{D}^n$.
\end{theorem}
\begin{remark}
     First, note that due to the privacy amplification, here the noise added to each sample gradient is $\tilde{O}(\frac{\beta M+\lambda}{\sqrt{n}\epsilon})$ rather than $\tilde{O}(\frac{\beta M+\lambda}{\epsilon})$ if without shuffling. Secondly, note that the truncation step is quite different from the previous work on DP-SCO with heavy-tailed data \cite{wang2020differentially}, i.e., we enforce the sample gradient to become zero if its norm exceeds the threshold. Finally, compared to the best-known result $O(\sqrt{\frac{\kappa}{n}})$ in the non-private and heavy-tailed case \cite{nazin2019algorithms} and the bound $\tilde{O}(\sqrt{\frac{\kappa}{n}}+\frac{\kappa \sqrt{d}}{n\epsilon})$ for private and Lipschitz case \cite{bassily2021non}, we can see there may exist a space to improve our bound further.
\end{remark}
There are two limitations in Theorem \ref{thm:10}. First, 
Algorithm \ref{alg4} is $(\epsilon, \delta)$ only for $\epsilon=\tilde{O}(n^{-\frac{1}{2}})$, which cannot be generalized to mid or low privacy regime. Secondly, Theorem \ref{thm:10} only holds for the case $1<p< 2$. To address the first issue, we can slightly modify the algorithm by using batched Mirror Descent without shuffling, while we will get a worse  upper bound. For the second one, similar to Theorem \ref{thm:8}, we can reduce the problem to the Euclidean case. Informally, we have the following two results.  
\begin{theorem}[Informal] \label{thm:11}
    For the $\ell_p^d$ space with $1<p<2$, suppose Assumption \ref{ass:4} holds. For all $0<\epsilon, \delta <1$, there is an $(\epsilon, \delta)$-DP algorithm  whose output could achieve an excess population risk of $O\left(M^\frac{4}{3}\frac{\kappa^\frac{2}{3}(d\log (1/\delta))^\frac{1}{6} }{(n\epsilon)^\frac{1}{3}}\right)$. 
\end{theorem}
\begin{theorem}[Informal]\label{thm:12} 
       For the $\ell_p^d$ space with $2\leq p\leq \infty$, suppose Assumption \ref{ass:4} holds (and with some additional  mild assumptions).  For all $0<\epsilon, \delta <1$, there is an $(\epsilon, \delta)$-DP algorithm  whose output could achieve an expected excess population risk of $O\left(\frac{d^{\frac{3}{2}-\frac{1}{p}} }{\sqrt{n}}+\frac{d^{\frac{3}{2}-\frac{1}{2p}} }{\sqrt{n\epsilon}}\right)$. 
\end{theorem}
\section*{Acknowledgements}
Di Wang was supported in part by the baseline funding BAS/1/1689-01-01, funding from the CRG grand URF/1/4663-01-01, FCC/1/1976-49-01 from CBRC.  He was also supported by the funding of the SDAIA-KAUST Center of Excellence in Data Science and Artificial Intelligence (SDAIA-KAUST AI).
\bibliographystyle{plain}
\bibliography{uai2023-template}
\newpage 
\appendix 
\section{Omitted Proofs in Section \ref{sec:euclidean}}
\subsection{Proof of Theorem \ref{thm:1}}
\begin{algorithm}
	\caption{$\mathcal{A}_{\text{ObjP}}$: Objective perturbation }
	\begin{algorithmic}[1]
		\State {\bfseries Input:} Dateset $D$, loss function $\ell$, regularization parameter $\lambda$.

	\State  Sample $\mathbf{G}\sim\mathcal{N}(0, \sigma_1^2\mathbb{I}_d)$ where $\sigma_1^2  = \frac{32L^2\log(1/\delta)}{\epsilon^2}$. Set $\lambda\geq \frac{r\beta}{2\epsilon n}$, where $r = \min\{d, 2\cdot \text{rank}(\nabla^2\ell(\theta, x))\}$ with $\text{rank}(\nabla^2\ell(\theta,x))$ being the maximal rank of the Hessian of $\ell$ for all $\theta \in \mathcal{C}$ and $x\sim \mathcal{P}$.
	
	\State Let $\mathcal{J}(\theta,D) = \hat{\mathcal{L}}+\frac{\langle \mathbf{G},\theta\rangle}{n}+\lambda||\theta||_2^2$.\\
	 \Return $\theta_1 = \underset{\theta\in\mathcal{C}}{\arg\min}\mathcal{J}(\theta, D)$.
	\end{algorithmic}
	\label{alg5}
\end{algorithm}
\begin{proof}
Let $\theta_1 = \underset{\theta\in\mathcal{C}}{\arg\min} \mathcal{J}(\theta, D)$, where $\mathcal{J}(\theta, D)=\hat{\mathcal{L}}(\theta, D)+\frac{\langle \mathbf{G},w\rangle}{n}+\lambda ||\theta||_2^2$. Let $\theta_2 = \mathcal{O}(\mathcal{J},\alpha)$ where $\mathcal{O}$ is the optimizer defined in the algorithm. Notice that one can compute $\hat{\theta}$ from tuple $(\theta_1, \theta_2-\theta_1+\mathbf{H})$ by simple post-processing. Furthermore, the algorithm that outputs $\theta_1$ is $(\epsilon,\delta)$-DP by the following theorem.
\begin{lemma}[Theorem 1 in \cite{iyengar2019towards}]
Suppose Assumption \ref{ass1} holds and that the smoothness parameter satisfy $\beta \leq \frac{\epsilon n \lambda}{r}$, the algorithm $\mathcal{A}_{\text{ObjP}}$ (Algorithm \ref{alg5}) that outputs $\theta_1=\underset{\theta\in\mathcal{C}}{\arg\min} \mathcal{J}(\theta,D)$ is $(\epsilon, \delta)$-DP.
\end{lemma}

Next, we will bound the term $||\theta_2 -\theta_1||$ to make $(\theta_2-\theta_1+\mathbf{H})$ differentially private, conditioned on $\theta_1$. As $\mathcal{J}(\theta, D)$ is $\lambda$-strongly convex, we have $\mathcal{J}(\theta_2, D)\geq \mathcal{J}(\theta_1,D) +\frac{\lambda}{2}||\theta_2-\theta_1||_2^2$, which implies that
\begin{equation}\label{eq13}
       ||\theta_2-\theta_1||_2\leq\sqrt{\frac{2}{\lambda}(\mathcal{J}(\theta_2,D)-\mathcal{J}(\theta_1, D))}\leq \sqrt{\frac{2\alpha}{\lambda}}.
\end{equation}
Thus, conditioned on $\theta_1$, $\theta_2 - \theta_1$ has the $l_2$ sensitivity of $\sqrt{\frac{8\alpha}{\lambda}}$. Therefore, $(\theta_2-\theta_1)+\mathbf{H}$ is $(\epsilon/2, \delta.2)$-DP. By the standard composition in \cite{dwork2014algorithmic}, the tuple $(\theta_1, \theta_2-\theta_1+\mathbf{H})$ satisfies $(\epsilon,\delta)$-DP and hence $\hat{\theta}$ satisfies $(\epsilon,\delta)$-DP.
\end{proof}
\subsection{Proof of Theorem \ref{th4}}
\begin{proof}
Let $\theta_1$ be the exact minimizer of $\mathcal{J}(\theta,D)$. We split the objective $\mathbb{E}[\mathcal{L}(\hat{\theta})]-\mathcal{L}(\theta^{*})$ into two parts and bound them separately.
\begin{equation}\label{eq1}
    \mathbb{E}[\mathcal{L}(\hat{\theta})] - \mathcal{L}(\theta^{*}) = \mathbb{E}[\mathcal{L}(\hat{\theta}) -\mathcal{L}(\theta_1)] + \mathbb{E}[\mathcal{L}(\theta_1)]-\mathcal{L}(\theta^{*}).
\end{equation}
In the following, we bound the term $\mathbb{E}[\mathcal{L}(\hat{\theta}) -\mathcal{L}(\theta_1)]$ and the term $\mathbb{E}[\mathcal{L}(\theta_1)]-\mathcal{L}(\theta^{*})$ separately. 
To bound the term $\mathbb{E}[\mathcal{L}(\theta_1)]-\mathcal{L}(\theta^{*})$, we need the following two lemmas. The first lemma states the excess empirical risk of $\theta_1$ while the second lemma states the stability property of regularized empirical risk minimization.
\begin{lemma}\label{le1}(Excess empirical loss of $\theta_1$ in $\mathcal{A}_{\text{ObjP}}$).
Let $D\sim\mathcal{P}^n$, under Assumption \ref{ass1}, the excess empirical loss of $\theta_1$ satisfies
\begin{equation}\label{eq12}
\mathbb{E}[\hat{\mathcal{L}}(\theta_1,D)]-\underset{\theta\in \mathcal{C}}{\min} \hat{\mathcal{L}}(\theta,D)\leq O\left( \frac{LG_{\mathcal{C}}\sqrt{\log (1/\delta)}}{\epsilon n} +\lambda ||\mathcal{C}||_2^2
\right),
\end{equation}
where the expectation is taken over the randomness induced by Gaussian noise.
\end{lemma}
\begin{lemma}\label{le2}[\cite{shalev2014understanding}]
Let $f:\mathcal{C}\times D\rightarrow \mathbb{R}$ be a convex, $\rho$-Lipschitz loss function where $D=\{x_1,\cdots,x_n\}\sim\mathcal{P}^n$.  Let $\mathcal{A}$ be an algorithm that outputs $\tilde{\theta} = \underset{\theta\in\mathcal{C}}{\arg\min} \{\hat{F}(\theta, D)+\lambda ||\theta||^2\}$  with $\lambda>0$ where $\hat{F}(\theta, D)= \frac{1}{n}\sum_{i =1}^n f(\theta, x_i)$, then $\mathcal{A}$ is $\frac{2\rho^2 }{\lambda n}$-uniformly stable, i.e., for all neighboring datasets $D\sim D'$ we have 
\begin{equation*}
    \sup_{z}|\mathbb{E}[f(\mathcal{A}(D),  z)- f(\mathcal{A}(D'),  z)]|\leq \frac{2\rho^2 }{\lambda n}. 
\end{equation*}
\end{lemma}
The property of uniform stability is described  by the following lemma.
\begin{lemma}[\cite{bousquet2002stability}]\label{le6}
Let $\mathcal{A}:\mathcal{X}^n\rightarrow \mathcal{C}$ be an $\alpha$-uniformly stable algorithm w.r.t. loss $\ell: \mathcal{C}\times \mathcal{X}\rightarrow \mathbb{R}$. Let $D\sim\mathcal{P}^n$ where $\mathcal{P}$ is the distribution over $\mathcal{X}$. Then,
\begin{equation*}
\underset{D\sim\mathcal{P}^n,\mathcal{A}}{\mathbb{E}}[\mathcal{L}(\mathcal{A}(D))-\hat{\mathcal{L}}(\mathcal{A}(D),D)]\leq \alpha.
\end{equation*}
\end{lemma}
Now we begin to bound the term $\mathcal{L}(\theta_1)-\mathcal{L}(\theta^{*})$ using the above three lemmas.
Fix any realization of the noise vector $\mathbf{G}$, we define $f_{\mathbf{G}}(\theta, x) =\ell(\theta, x) +\frac{\langle \mathbf{G},\theta\rangle}{n}$, then $f_{\mathbf{G}}$ is $\left(L+\frac{||\mathbf{G}||_2}{n}\right)$-Lipschitz.

Define $\hat{F}_{\mathbf{G}}(\theta, D) =\frac{1}{n}\sum_{i =1}^n f_{\mathbf{G}}(\theta, x_i)$, and we have $\theta_1 = \underset{\theta\in\mathcal{C}}{\arg\min} \hat{F}_{\mathbf{G}}(\theta, D) +\lambda ||\theta||_2^2$, so from Lemma \ref{le2}, the algorithm that outputs $\theta_1$ is $\frac{2\left(L+\frac{||\mathbf{G}||_2}{n}\right)^2}{\lambda  n }$-uniformly stable.
Denote $F_{\mathbf{G}}(\theta) = \underset{x\sim\mathcal{P}}{\mathbb{E}}[f_{\mathbf{G}}(\theta, x)]$, according to Lemma \ref{le6}, we have 
\begin{equation*}
    \underset{D\sim\mathcal{P}^n}{\mathbb{E}}[\mathcal{L}(\theta)-\hat{\mathcal{L}}(\theta, D)] = \underset{D\sim\mathcal{P}^n}{\mathbb{E}}[F_{\mathbf{G}}(\theta)-\hat{F}_{\mathbf{G}}(\theta,D)]\leq \frac{2\left(L+\frac{||\mathbf{G}||_2}{n}\right)^2}{\lambda n}.
\end{equation*}
Take the expectation w.r.t. $\mathbf{G}\sim \mathcal{N}(0, \frac{32L^2\log(1/\delta)}{\epsilon^2}\mathbb{I}_d)$ as well, we get 
\begin{equation}\label{eq10}
\begin{aligned}
\mathbb{E}[\mathcal{L}(\theta)-\hat{\mathcal{L}}(\theta, D)]\leq O\left(\frac{L^2\cdot \left(1+\frac{\sqrt{d\log(1/\delta)}}{\epsilon n}\right)^2}{\lambda n}\right)\leq O\left(\frac{L^2}{\lambda n}\right),
\end{aligned}
\end{equation}
 where we assume $n\geq O(\frac{\sqrt{d \log(1/\delta)}}{\epsilon})$.
 
Thus
\begin{equation}\label{eq14}
    \begin{aligned}
    \mathbb{E}[\mathcal{L}(\theta_1)] -\mathcal{L}(\theta^{*})& =  \mathbb{E}[\mathcal{L}(\theta_1)] -\underset{\theta\in\mathcal{C}}{\min} \mathcal{L}(\theta)\\
    &\leq \mathbb{E}[\hat{\mathcal{L}}(\theta_1,D)-\underset{\theta\in\mathcal{C}}{\min}\hat{\mathcal{L}}(\theta, D)] + \mathbb{E}[\mathcal{L}(\theta_1)-\hat{\mathcal{L}}(\theta_1,D)]\\
    & \leq O\left(\frac{L\cdot G_{\mathcal{C}}\cdot\sqrt{\log(1/\delta)}}{\epsilon n}+\lambda ||\mathcal{C}||_2^2 +\frac{L^2}{\lambda n}
    \right),
    \end{aligned}
\end{equation}
where we use the fact that $\underset{D\sim \mathcal{P}^n}{\mathbb{E}}[\underset{\theta\in \mathcal{C}}{\min \hat{\mathcal{L}}}(\theta, D)]\leq \underset{\theta \in \mathcal{C}}{\min}\underset{D\sim \mathcal{P}^n}{\mathbb{E}}[\hat{\mathcal{L}}(\theta,D)] =\underset{\theta \in\mathcal{C}}{\min} \mathcal{L}(\theta)$ and the last bound is directly from Eq.(\ref{eq12}) and Eq.(\ref{eq10}).

Now we bound the term $\mathbb{E}[\mathcal{L}(\hat{\theta})]-\mathcal{L}(\theta_1)$.
Recall that $\theta_2 =\mathcal{O}(\mathcal{J},\alpha)$ and 
\begin{equation*}
    \mathbb{E}[\mathcal{L}(\hat{\theta})]-\mathcal{L}(\theta_1) = \mathbb{E}[\mathcal{L}(\hat{\theta})] -\mathcal{L}(\theta_2) +\mathcal{L}(\theta_2)-\mathcal{L}(\theta_1).
\end{equation*}

Note the term $\mathcal{L}(\theta_2)-\mathcal{L}(\theta_1)\leq L\cdot ||\theta_1-\theta_2||_2\leq L\cdot \sqrt{\frac{2\alpha}{\lambda}}$ (From Eq.(\ref{eq13})), and the term $\mathbb{E}[\mathcal{L}(\hat{\theta})] -\mathcal{L}(\theta_2)\leq L\cdot\mathbb{E}[||\hat{\theta}-\theta_2||_2]$. 

Also note that $\hat{\theta} =\text{Proj}_{\mathcal{C}}(\theta_2+\mathbf{H})$. Let $q$ be the line through $\theta_2$ and $\hat{\theta}$, and let $p$ be the projection of $\theta_3 = \theta_2+\mathbf{H}$ onto $q$. The key observation is that $p$ lies on the ray from $\hat{\theta}$ to infinity otherwise $p$ will be a point in $\mathcal{C} $ that is closer to $\theta_3 $ than $\hat{\theta}$. Thus we have
\begin{equation*}
    \begin{aligned}
    \mathbb{E}[||\hat{\theta}-\theta_2||_2^2]&= \mathbb{E}[\langle \hat{\theta}-\theta_2 ,\hat{\theta}-\theta_2\rangle]\\
    & \leq \mathbb{E}[\langle \hat{\theta}-\theta_2,\theta_3-\theta_2]\\
    & =\mathbb{E}[\langle \mathbf{H},\hat{\theta}-\theta_2\rangle]\\
    &\leq 2\cdot\underset{\theta\in\mathcal{C}}{\max}\mathbb{E}[\langle \mathbf{H}, \theta\rangle]\\
    & \leq O(\mathbb{E}[\max |\langle \mathbf{H},\theta\rangle|])\\
    & =O\left(\sqrt{\frac{\alpha \log(1/\delta)}{\lambda}}\cdot \frac{G_{\mathcal{C}}}{\epsilon}\right),
    \end{aligned}
\end{equation*}
where the last equation is from the definition of Gaussian width.

So we have 
    \begin{equation}\label{eq15}
    \begin{aligned}
    \mathbb{E}[\mathcal{L}(\hat{\theta})]-\mathcal{L}(\theta_1)&\leq   
     L\cdot \sqrt{\frac{2\alpha}{\lambda}}+L\cdot \mathbb{E}[||\hat{\theta}-\theta_2||_2]\\
    &\leq O\left(L\cdot\sqrt[4]{\frac{\alpha \log(1/\delta)}{\lambda}}\cdot \sqrt{\frac{G_{\mathcal{C}}}{\epsilon}}+L\sqrt{\frac{\alpha}{\lambda}}\right).
    \end{aligned}
    \end{equation}
    In total, combining Eq.(\ref{eq14}) and Eq.(\ref{eq15}), we can bound Eq. (\ref{eq1}) by
    \begin{equation*}
    \begin{aligned}
        \mathbb{E}[\mathcal{L}(\hat{\theta})]- \mathcal{L}(\theta^{*})
        &=\mathbb{E}[\mathcal{L}(\hat{\theta}) -\mathcal{L}(\theta_1)] + \mathbb{E}[\mathcal{L}(\theta_1)]-\mathcal{L}(\theta^{*})\\
        &\leq O\left( L\cdot\sqrt[4]{\frac{\alpha \log(1/\delta)}{\lambda}}\cdot \sqrt{\frac{G_{\mathcal{C}}}{\epsilon}}+L\sqrt{\frac{\alpha}{\lambda}}+\frac{L\cdot G_{\mathcal{C}}\cdot\sqrt{\log(1/\delta)}}{\epsilon n}+\lambda ||\mathcal{C}||_2^2 +\frac{L^2}{\lambda n}\right).
        \end{aligned}
    \end{equation*}
    Since $\alpha\leq \min\left\{
\frac{L||\mathcal{C}||_2}{n^{\frac{3}{2}}}, \frac{\epsilon^2 L||\mathcal{C}||_2^3}{G^2_{\mathcal{C}}\log(1/\delta) n^{\frac{5}{2}}}
\right\}$, we have $\sqrt{L\cdot||\mathcal{C}||_2\sqrt{n}\alpha}\leq \frac{L\cdot ||\mathcal{C}||_2}{\sqrt{n}}$ and $L\cdot\sqrt[4]{\frac{\alpha \log(1/\delta)}{\lambda}}\cdot \sqrt{\frac{G_{\mathcal{C}}}{\epsilon}}\leq \frac{L\cdot ||\mathcal{C}||_2}{\sqrt{n}}$. Let $\lambda = \frac{L}{\sqrt{n}||\mathcal{C}||_2}$,
    then 
    \begin{equation*}
        \mathbb{E}[\mathcal{L}(\hat{\theta})]-\mathcal{L}(\theta^{*}) \leq O\left(\frac{L\cdot G_{\mathcal{C}}\cdot\sqrt{\log(1/\delta)}}{\epsilon n}+\frac{L||\mathcal{C}||_2}{\sqrt{n}}\right).
    \end{equation*}
    Note that we need $\lambda =\frac{L}{\sqrt{n}||\mathcal{C}||_2}\geq \frac{r \beta}{\epsilon n }$, namely, $n\geq \frac{r^2\beta^2 ||\mathcal{C}||_2^2}{\epsilon^2 L^2}$.
\end{proof}
\begin{proof}[{\bf Proof of Lemma \ref{le1}}]
Let $\bar{\mathcal{L}}(\theta, D) = \hat{\mathcal{L}}(\theta,D) + \lambda ||\theta||_2^2$ and $\bar{\theta}=\underset{\theta\in\mathcal{C}}{\arg\min} \bar{\mathcal{L}}(\theta, D)$. So $\mathcal{J}(\theta, D) = \bar{\mathcal{L}}(\theta, D) +\frac{\langle \mathbf{G}, \theta\rangle}{n}$. Since $\theta_1 $ minimizes $ \mathcal{J}(\theta, D)$, we have $\mathcal{J}(\bar{\theta}, D)\geq \mathcal{J}(\theta_1, D)$, namely,
\begin{equation*}
    \bar{\mathcal{L}}(\bar{\theta}, D) +\frac{\langle \mathbf{G}, \bar{\theta}\rangle}{n} \geq \bar{\mathcal{L}}(\theta_1,D) +\frac{\langle \mathbf{G},\theta_1\rangle}{n}.
\end{equation*}
Recall that $\mathbf{G}\sim \mathcal{N}(0, \frac{128L^2\log(1/\delta)}{\epsilon^2} \mathbb{I}_d)$, rearrange the inequality and take the expectation at both sides and we get
\begin{equation*}
    \begin{aligned}
      \mathbb{E}[\bar{\mathcal{L}}(\theta_1,D)-\bar{\mathcal{L}}(\bar{\theta},D)]&\leq \mathbb{E}[\frac{\langle\mathbf{G},\bar{\theta}-\theta_1\rangle}{n}]\\
      & \leq 2 \cdot\underset{\theta\in\mathcal{C}}{\max}\mathbb{E} \left[\frac{\langle \mathbf{G}, \theta\rangle}{n}\right]\\
      &\leq 2\cdot \mathbb{E}\left[\underset{\theta\in\mathcal{C}}{\max}\left|\frac{\langle \mathbf{G},\theta\rangle}{n}\right|\right]\\
      &=O\left(\frac{L\cdot G_{\mathcal{C}}\sqrt{\log(1/\delta)}}{\epsilon n}\right),
    \end{aligned}
\end{equation*}
where the last bound is from the definition of Gaussian width.

Thus 
\begin{equation*}
    \begin{aligned}
    \mathbb{E}[\hat{\mathcal{L}}(\theta_1,D)-\hat{\mathcal{L}}(\theta^{*},D)] &= \mathbb{E}[\bar{\mathcal{L}}(\theta_1, D) - \bar{\mathcal{L}}(\theta^{*},D)+\lambda ||\theta^{*}||_2^2 -\lambda ||\theta_1||_2^2]\\
    &\leq \mathbb{E}[\bar{\mathcal{L}}(\theta_1, D) - \bar{\mathcal{L}}(\theta^{*},D)+\lambda ||\theta^{*}||_2^2]\\
    & \leq \mathbb{E}[\bar{\mathcal{L}}(\theta_1, D) - \bar{\mathcal{L}}(\bar{\theta},D)+\lambda ||\theta^{*}||_2^2]\\
    & \leq O\left(\frac{L\cdot G_{\mathcal{C}}\sqrt{\log(1/\delta)}}{\epsilon n} +\lambda ||\mathcal{C}||_2^2\right).
    \end{aligned}
\end{equation*}
\end{proof}
\subsection{Proof of Theorem \ref{thm:3}}
\begin{proof}
The proof is similar to the convex case. Note that $\mathcal{J}(\theta, D)$ is a  $\frac{r\beta}{\epsilon n}$-strongly convex function.
\end{proof}
\subsection{Proof of Theorem \ref{th5}}
\begin{proof}
By the assumptions we made about $n$, we have $\Delta\geq \frac{L\cdot ||\mathcal{C}||_2}{\sqrt{n}}$ and $\frac{L}{\sqrt{n}||\mathcal{C}||_2}\geq \frac{r \beta}{\epsilon n}$. 

Since the loss function is $\Delta$-strongly convex with respect to $||\cdot||_{\mathcal{C}}$, which implies that the loss function is $\frac{\Delta}{||\mathcal{C}||_2^2}$-strongly convex w.r.t. $||\cdot||_2$ and thus  $\frac{L}{\sqrt{n}||\mathcal{C}||_2}$-strongly convex w.r.t. $||\cdot||_{2}$, where we use the fact that $\Delta\geq \frac{L\cdot ||\mathcal{C}||_2}{\sqrt{n}}$ and $||v||_{\mathcal{C}}\geq \frac{||v||_2}{||\mathcal{C}||_2}$ for any vector $v\in\mathcal{C}$.

Since $\Delta \geq \frac{L}{\sqrt{n}||\mathcal{C}||_2}\geq \frac{r \beta}{\epsilon n}
$, we have $\lambda= \max\left\{
	\frac{r \beta}{\epsilon n}-\Delta,0
	\right\}=0$.
 
The population loss can be disassembled as the following two parts, and we bound them separately. 
\begin{equation*}
    \mathbb{E}[\mathcal{L}(\hat{\theta})] - \mathcal{L}(\theta^{*}) = \mathbb{E}[\mathcal{L}(\hat{\theta})-\mathcal{L}(\theta_1)] + \mathbb{E}[\mathcal{L}(\theta_1)] - \mathcal{L}(\theta^{*}).
\end{equation*}
We first bound $\mathbb{E}[\mathcal{L}(\hat{\theta})-\mathcal{L}(\theta_1)]$.
Note that
\begin{equation*}
    \mathbb{E}[\mathcal{L}(\hat{\theta})-\mathcal{L}(\theta_1)] = \mathbb{E}[\mathcal{L}(\hat{\theta})-\mathcal{L}(\theta_2)] +\mathbb{E}[\mathcal{L}(\theta_2)-\mathcal{L}(\theta_1)].
\end{equation*}
For term $\mathbb{E}[\mathcal{L}(\theta_2)-\mathcal{L}(\theta_1)]$, since $\mathcal{L}$ is $\Delta$-strongly convex w.r.t. $||\cdot||_{\mathcal{C}}$ and thus $\frac{\Delta}{||\mathcal{C}||_2^2}$-strongly convex w.r.t. $||\cdot||_2$. So by the definition of strong convexity of $\mathcal{L}$, we have
\begin{equation*}
    \alpha \geq \mathcal{L}(\theta_2)-\mathcal{L}(\theta_1)\geq \frac{\Delta}{2||\mathcal{C}||_2^2}||\theta_2-\theta_1||_2^2,
\end{equation*}
where $\alpha$ is the optimization accuracy.

Thus,
\begin{equation*}
    ||\theta_2-\theta_1||_2\leq \sqrt{\frac{2\alpha ||\mathcal{C}||_2^2}{\Delta}}.
\end{equation*}

So using the definition of $L$-Lipschitz,
\begin{equation*}
    \mathbb{E}[\mathcal{L}(\theta_2)-\mathcal{L}(\theta_1)] \leq L \cdot \mathbb{E}[||\theta_2-\theta_1||_2]\leq L\cdot \sqrt{\frac{2\alpha ||\mathcal{C}||_2^2}{\Delta}}.
\end{equation*}

For term $\mathbb{E}[\mathcal{L}(\hat{\theta})-\mathcal{L}(\theta_2)]$, it is similar to the convex case, and we have
\begin{equation*}
    \mathbb{E}[\mathcal{L}(\hat{\theta})-\mathcal{L}(\theta_2)]\leq O\left(L\cdot\sqrt[4]{\frac{\alpha \log(1/\delta)||\mathcal{C}||_2^2}{\Delta}}\cdot \sqrt{\frac{G_{\mathcal{C}}}{\epsilon}}\right).
\end{equation*}
Thus, 
\begin{equation*}
    \mathbb{E}[\mathcal{L}(\hat{\theta})-\mathcal{L}(\theta_1)]\leq O\left(L\cdot\sqrt[4]{\frac{\alpha \log(1/\delta)||\mathcal{C}||_2^2}{\Delta}}\cdot \sqrt{\frac{G_{\mathcal{C}}}{\epsilon}} + L\cdot \sqrt{\frac{2\alpha ||\mathcal{C}||_2^2}{\Delta}}
    \right).
\end{equation*}
Next we bound $\mathbb{E}[\mathcal{L}(\theta_1)] - \mathcal{L}(\theta^{*})$.
Note that 
\begin{equation*}
   \mathbb{E} [\mathcal{L}(\theta_1)]-\mathcal{L}(\theta^{*})\leq \mathbb{E}[\hat{\mathcal{L}}(\theta_1,D)-\underset{\theta\in\mathcal{C}}{\min}\hat{\mathcal{L}}(\theta, D)] + \mathbb{E}[\mathcal{L}(\theta_1)-\hat{\mathcal{L}}(\theta_1,D)],
\end{equation*}
where we used the fact that $\mathbb{E}[\underset{\theta \in \mathcal{C}}{\min }\hat{\mathcal{L}}(\theta, D)]\leq \underset{\theta \in \mathcal{C}}{\min } \mathbb{E}[\hat{\mathcal{L}}(\theta, D)]=\mathcal{L}(\theta^{*})$.

For term $\mathbb{E}[\mathcal{L}(\theta_1)-\hat{\mathcal{L}}(\theta_1,D)]$, note that with $\lambda = 0$,  $f_{\mathbf{G}}(\theta, x) = \ell(\theta, x)+\frac{\langle \mathbf{G},\theta\rangle}{n}$ would be $\frac{\Delta}{||\mathcal{C}||_2^2}$ strongly convex w.r.t. $||\cdot||_2$. Using the same notation as in the convex case, where $\hat{F}_{\mathbf{G}}(\theta,D) = \frac{1}{n}\sum_{i=1}^n f_{\mathbf{G}}(\theta,x_i)$ and $F_{\mathbf{G}}(\theta)=\underset{x\sim \mathcal{P}}{\mathbb{E}}[f_{\mathbf{G}}(\theta,x)]$, we have
\begin{equation*}
\begin{aligned}
    \mathbb{E}[\mathcal{L}(\theta_1)-\hat{\mathcal{L}}(\theta_1,D)]& = \mathbb{E}[F_{\mathbf{G}}(\theta_1)-\hat{F}_{\mathbf{G}}(\theta_1,D)]\\
    &\leq \frac{\left(L+\frac{||\mathbf{G}||_2}{n}\right)^2||\mathcal{C}||_2^2}{n\Delta}~(\text{According to Lemma \ref{le2}})\\
    &\leq O\left(\frac{L^2||\mathcal{C}||_2^2}{n\Delta}
    \right) (\text{since} ~n\geq O\left(\frac{\sqrt{d\log(1/\delta)}}{\epsilon}\right)).
    \end{aligned}
\end{equation*}
Let $\theta^{'} = \underset{\theta\in\mathcal{C}}{\arg\min}\hat{\mathcal{L}}(\theta, D)$. In the following, we bound 
the term $\mathbb{E}[\hat{\mathcal{L}}(\theta_1,D)-\underset{\theta\in\mathcal{C}}{\min}\hat{\mathcal{L}}(\theta, D)]=\mathbb{E}[\hat{\mathcal{L}}(\theta_1,D)-\hat{\mathcal{L}}(\theta^{'}, D)]$.

By the definition of strong convexity,
\begin{equation*}
\begin{aligned}
   \hat{\mathcal{L}}(\theta_1,D)&\geq \hat{\mathcal{L}}(\theta^{'},D)+\frac{\Delta}{2} ||\theta_1-\theta^{'}||_{\mathcal{C}}^2,\\
   \Leftrightarrow \hat{\mathcal{L}}(\theta_1,D)+\frac{\langle \mathbf{G},\theta_1\rangle}{n}-\frac{\langle \mathbf{G},\theta_1\rangle}{n}&\geq \hat{\mathcal{L}}(\theta^{'},D)+\frac{\langle \mathbf{G},\theta^{'}\rangle}{n}-\frac{\langle \mathbf{G},\theta^{'}\rangle}{n}+\frac{\Delta}{2} ||\theta_1-\theta^{'}||_{\mathcal{C}}^2, \\
   \Leftrightarrow \mathcal{J}(\theta_1,D)-\frac{\langle \mathbf{G},\theta_1\rangle}{n}&\geq \mathcal{J}(\theta^{'},D)-\frac{\langle \mathbf{G},\theta^{'}\rangle}{n}+\frac{\Delta}{2} ||\theta_1-\theta^{'}||_{\mathcal{C}}^2.
   \end{aligned}
\end{equation*}
So,
\begin{equation*}
     \mathcal{J}(\theta_1,D)-\mathcal{J}(\theta^{'},D)+\frac{\langle \mathbf{G},\theta^{'}-\theta_1\rangle}{n}\geq \frac{\Delta}{2} ||\theta_1-\theta^{'}||_{\mathcal{C}}^2.
\end{equation*}
Since $\mathcal{J}(\theta_1,D)-\mathcal{J}(\theta^{'},D)\leq 0$ (due to the optimality condition), we get
\begin{equation}\label{eq16}
\begin{aligned}
    \frac{\langle \mathbf{G},\theta^{'}-\theta_1\rangle}{n}&\geq \frac{\Delta}{2} ||\theta_1-\theta^{'}||_{\mathcal{C}}^2,\\
    \Rightarrow &||\theta_1-\theta^{'}||_{\mathcal{C}}\leq \frac{2\cdot \langle \mathbf{G},\frac{\theta^{'}-\theta_1}{||\theta^{'}-\theta_1||_{\mathcal{C}}}\rangle }{n\Delta},
   \\
     \Rightarrow & ||\theta_1-\theta^{'}||_{\mathcal{C}}\leq 2\cdot \underset{\theta \in \mathcal{C}} {\max}\frac{\langle \mathbf{G},\theta\rangle}{n\Delta} = \frac{2||\mathbf{G}||_{\mathcal{C}^{*}}}{n\Delta}.
    \end{aligned}
\end{equation}
Using $\mathcal{J}(\theta_1,D)-\mathcal{J}(\theta^{'},D)\leq 0$ again, and take the expectation at both sizes, 
\begin{equation*}
    \mathcal{L}(\theta^{'})+\mathbb{E}[\frac{\langle \mathbf{G},\theta^{'}\rangle}{n}] \geq \mathcal{L}(\theta_1)+ \mathbb{E}[\frac{\langle\mathbf{G},\theta_1\rangle}{n}].
\end{equation*}
Thus
\begin{equation*}
\begin{aligned}
    \mathcal{L}(\theta_1)-\mathcal{L}(\theta^{'})&
    \leq \mathbb{E}[\frac{\langle \mathbf{G},\theta^{'}-\theta_1\rangle}{n}]\\
    &\leq \mathbb{E}\left[\frac{||\mathbf{G}||_{\mathcal{C}^{*}}}{n}\cdot ||\theta_1 -\theta^{'}||_{\mathcal{C}}\right]~\text{(Holder’s inequality)}\\
    & \leq \mathbb{E}\left[\frac{2||\mathbf{G}||^2_{\mathcal{C}^{*}}}{n^2 \Delta}
   \right]~\text{(according to Eq.(\ref{eq16}))}\\
    &\leq O\left(
    \frac{G_{\mathcal{C}}^2L^2 \log(1/\delta)}{\Delta n^2 \epsilon^2}
    \right).
    \end{aligned}
\end{equation*}
Thus $\mathbb{E}[\hat{\mathcal{L}}(\theta_1,D)-\underset{\theta\in\mathcal{C}}{\min}\hat{\mathcal{L}}(\theta, D)]\leq  O\left(\frac{L^2||\mathcal{C}||_2^2}{n\Delta}+
    \frac{G_{\mathcal{C}}^2L^2 \log(1/\delta)}{\Delta n^2 \epsilon^2}
    \right)$.
    So
    \begin{equation*}
         \mathbb{E}[\mathcal{L}(\hat{\theta})] - \mathcal{L}(\theta^{*})\leq O\left(
         \frac{L^2||\mathcal{C}||_2^2}{n\Delta}+
    \frac{G_{\mathcal{C}}^2L^2 \log(1/\delta)}{\Delta n^2 \epsilon^2}
        + L\cdot\sqrt[4]{\frac{\alpha \log(1/\delta)||\mathcal{C}||_2^2}{\Delta}}\cdot \sqrt{\frac{G_{\mathcal{C}}}{\epsilon}} + L\cdot \sqrt{\frac{2\alpha ||\mathcal{C}||_2^2}{\Delta}}
         \right).
    \end{equation*}
    When $\alpha\leq O\left(\min\left\{
\frac{L^2 ||\mathcal{C}||_2^2}{\Delta n^2},\frac{L^4\cdot ||\mathcal{C}||_2^6\epsilon^2}{\Delta^3 n^4 G_{\mathcal{C}}^2\log(1/\delta)}
\right\}
\right)$, we have $ L\cdot \sqrt{\frac{2\alpha ||\mathcal{C}||_2^2}{\Delta}}\leq  \frac{L^2||\mathcal{C}||_2^2}{n\Delta} $ and $ L\cdot\sqrt[4]{\frac{\alpha \log(1/\delta)||\mathcal{C}||_2^2}{\Delta}}\cdot \sqrt{\frac{G_{\mathcal{C}}}{\epsilon}}\leq  \frac{L^2||\mathcal{C}||_2^2}{n\Delta}$.

Thus,
\begin{equation*}
    \mathbb{E}[\mathcal{L}(\hat{\theta})] - \mathcal{L}(\theta^{*})\leq O\left( \frac{L^2||\mathcal{C}||_2^2}{n\Delta} +  \frac{G_{\mathcal{C}}^2L^2 \log(1/\delta)}{\Delta n^2 \epsilon^2}
    \right).
\end{equation*}

\end{proof}

\subsection{\bf Proof of Theorem \ref{th1}}
\begin{proof}
To show the proof, we first prove  the following theorem on the lower bound of excess empirical risk and then use reduction  from Private ERM to Private SCO to get the lower bound for excess population risk.
\begin{theorem}\label{th2}
Let $\mathcal{C}$ be a symmetric body contained in the unit Euclidean ball $\mathcal{B}_2^d$ in $\mathbb{R}^d$ and satisfies $\|\mathcal{C}\|_2=1$. For any $ n=O(\frac{\sqrt{d\log(1/\delta)}}{\epsilon})$, $\epsilon=O(1)$ and  $2^{-\Omega(n)}\leq \delta\leq 1/n^{1+\Omega(1)}$, there exists a loss $\ell$ which is $1$-Lipschitz w.r.t. $\|\cdot\|_2$ and $\mathcal{C}^2_{\min}$-strongly convex w.r.t. $\|\cdot\|_\mathcal{C}$, and a dataset $D=\{x_1,\cdots,x_n\}\subseteq \mathcal{C}$ such as for any $(\epsilon,\delta)$-differentially private algorithm $\mathcal{A}$, its output 
  satisfies 
	$$
	\mathbb{E}[\hat{\mathcal{L}}(\mathcal{A},D)]-\underset{\theta\in\mathcal{C}}{\min} \hat{\mathcal{L}}(\theta,D)=\Omega \left(\frac{G_{\mathcal{C}}^2\log(1/\delta)}{(\log(2d))^4\epsilon^2n^2}\right),$$ 
  where the expectation is taken over the internal randomness of the algorithm $\mathcal{A}$.
\end{theorem}
\begin{theorem}[Reduction from private ERM to private SCO \cite{bassily2019private}]\label{th8}
For any $\gamma>0$, suppose there is a $\left(\frac{\epsilon}{4 \log(1/\delta)}, \frac{e^{-\epsilon}\delta}{8\log (2/\delta)}\right)$-DP algorithm $\mathcal{A}$ such that for any distribution on domain $\mathcal{X}$, $\mathcal{A}$  yields expected population loss $\mathbb{E}_{\mathcal{A}}[\mathcal{L}(\mathcal{A})]-\min_w \mathcal{L}(w)< \gamma$. Then, there is a $(\epsilon,\delta)$-DP algorithm $\mathcal{B}$ that given any dataset  $D\in \mathcal{X}^n$, it yields expected excess empirical loss $\mathbb{E}_{\mathcal{B}}[\hat{\mathcal{L}}(\mathcal{B},D)]-\min_w \hat{\mathcal{L}}(w,D)< \gamma$.

\end{theorem}
From Theorem \ref{th8}, for any dataset $D$ and any 1-Lipschitz, $\mathcal{C}_{\min}^2$- strongly convex loss $\ell$, if there exists an algorithm with excess population loss
\begin{equation*}
	\mathbb{E}[{\mathcal{L}}(\theta^{priv})]-\underset{\theta\in\mathcal{C}}{\min} {\mathcal{L}}(\theta)=o \left(\frac{G_{\mathcal{C}}^2\log(1/\delta)}{(\log(2d))^4\epsilon^2n^2}\right),
\end{equation*}
then there exists an algorithm $\mathcal{B}$ such that the excess empirical loss $
	\mathbb{E}[\hat{\mathcal{L}}(\mathcal{B},D)]-\underset{\theta\in\mathcal{C}}{\min} \hat{\mathcal{L}}(\theta,D)=o \left(\frac{G_{\mathcal{C}}^2\log(1/\delta)}{(\log(2d))^4\epsilon^2n^2}\right)
	$, which contradicts Theorem  \ref{th2}.

Thus, $\forall n=O(\frac{\sqrt{d\log(1/\delta)}}{\epsilon})$, there exists a dataset $D=\{x_1,\cdots,x_n\}\subseteq \mathcal{C}$ and a strongly convex loss function $\ell$ such that for any output $\theta^{priv}$, the excess population loss
	$
	\mathbb{E}[{\mathcal{L}}(\theta^{priv})]-\underset{\theta\in\mathcal{C}}{\min} {\mathcal{L}}(\theta)=\Omega \left(\frac{G_{\mathcal{C}}^2\log(1/\delta)}{(\log(2d))^4\epsilon^2n^2}\right)$.

As a result, we have
\begin{equation*}
	\mathbb{E}[{\mathcal{L}}(\theta^{priv})]-\underset{\theta\in\mathcal{C}}{\min} {\mathcal{L}}(\theta)=\Omega \left(\max\left\{\frac{G_{\mathcal{C}}^2\log(1/\delta)}{(\log(2d))^4\epsilon^2n^2},\frac{1}{n}\right\}\right),
\end{equation*}
where the first term is the lower bound on excess empirical loss and the second term is the lower bound on excess population loss in the non-private setting.
	
\end{proof}
\begin{proof}[{\bf Proof of Theorem \ref{th2}}]
Before starting our proof, we give some background on the mean point problem.

Let $\bar{x}=\frac{1}{n}\sum\limits_{i=1}^nx_i$ be the mean of the database $D$, where $D=\{x_1,\cdots,x_n\}$ is a multiset of points in $\mathcal{C}$.
The sample complexity of the mean point problem to achieve an error $\alpha$ with respect to an algorithm $\mathcal{A}$ is defined as 
\begin{equation*}
SC_{mp}(\mathcal{C},\mathcal{A},\alpha)=\min \{n:\underset{D}{\sup}~(\mathbb{E}||\mathcal{A}(D)-\bar{x}||_2^2)^{1/2}\leq \alpha\},
\end{equation*}
where the supremum is taken over the database $D$ consisting of at most $n$ points from $\mathcal{C}$ and the expectation is taken over the randomness of the algorithm $\mathcal{A}$.

The sample complexity of solving the mean point problem with error $\alpha$ under $(\epsilon,\delta)$-differential privacy over convex set $\mathcal{C}$ is defined as the minimum number of samples among all the differentially private algorithm $\mathcal{A}$.
\begin{equation*}
SC_{mp}(\mathcal{C},\alpha)=\min \{	SC_{mp}(\mathcal{C},\mathcal{A},\alpha):\mathcal{A} ~\text{is}~ (\epsilon,\delta)\text{-differentially private} \}.
\end{equation*}
Previous work \cite{kattis2016lower} shows that we can characterize sample complexity $SC_{mp}(\mathcal{C},\alpha)$ as a natural property of convex set $\mathcal{C}$.

\begin{lemma}\cite{kattis2016lower}\label{le7}
	Let $\mathcal{C}$ be a symmetric convex body contained in the unit Euclidean ball $\mathcal{B}_2^d$ in $\mathbb{R}^{d}$. Let $c$ be an absolute constant, then for any $\epsilon=O(1), 2^{-\Omega(n)}\leq \delta\leq 1/n^{1+\Omega(1)}$ and any $\alpha\leq \frac{G_{\mathcal{C}}}{c\sqrt{d}(\log2d)^2}$, 
	\begin{equation}\label{eq17}
	SC_{mp}(\mathcal{C},\alpha)=\Omega\left(\frac{G_{\mathcal{C}}\sqrt{\log(1/\delta)}}{(\log 2d)^2\alpha\epsilon}\right),
	\end{equation}
	\begin{equation*}
	SC_{mp}(\mathcal{C},\alpha)=O\left(\min\left\{\frac{G_{\mathcal{C}}\sqrt{\log(1/\delta)}}{\alpha^2\epsilon},\frac{\sqrt{d\log(1/\delta)}}{\alpha\epsilon}\right\}\right).
	\end{equation*}
	
	When $G_{\mathcal{C}}=\Omega(\sqrt{d})$, then $SC_{mp}(\mathcal{C},\alpha)=\Theta\left(\frac{\sigma(\epsilon, \delta)\sqrt{d}}{\alpha}\right)$ for any $\alpha\leq 1/c$.
\end{lemma}

Now we start our proof with the help of the above lemma.

	Let $\ell(\theta;x)=\frac{1}{2
 } ||\theta-x||_2^2$ be half of the squared $\ell_2$-distance between $\theta\in \mathcal{C}\subseteq \mathcal{B}_2^d$ and $x_i\in \mathcal{C}$, which is $1$-Lipschitz and $1$-strongly convex w.r.t to $\|\cdot\|_2$. Actually, based on the following lemma we can easily show it is $\mathcal{C}^2_{\min}$-strongly convex w.r.t $\|\cdot\|_\mathcal{C}$. 

 \begin{lemma}
     For any $x$, we have $\|x\|_2\geq \|x\|_\mathcal{C} \cdot \mathcal{C}_{\min}$.
 \end{lemma}
\begin{proof}
    By the definition of $ \|x\|_\mathcal{C}$ we can see it is sufficient to show that $x\in \frac{\|x\|_2}{\mathcal{C}_{\min} } \mathcal{C}$. Note that as $\mathcal{C}$ is symmetric and  $\mathcal{C}_{\min} $ is the minimal distance from the original point to the boundary of $\mathcal{C}$, thus, $\frac{\mathcal{C}}{\mathcal{C}_{\min} }$ contains the unit $\ell_2$-norm ball, indicating that $x\in \frac{\|x\|_2}{\mathcal{C}_{\min} } \mathcal{C}$. 
\end{proof}

	The strongly convex decomposable loss function is defined as $\hat{\mathcal{L}}(\theta;D)=\frac{1}{2n}\sum\limits_{i=1}^n\ell(\theta;x_i)=\frac{1}{2n}\sum\limits_{i=1}^n||\theta-x_i||_2^2$. 
	Notice that the minimizer of $\hat{\mathcal{L}}(\cdot;D)$ over $\mathcal{B}_2^d$ is $\theta^{*}=\frac{1}{n}\sum\limits_{i=1}^n x_i\in\mathcal{C}$, and the excess empirical risk can be written as:
	\begin{equation*}
	\mathbb{E}[\hat{\mathcal{L}}(\theta^{priv};D)]-\hat{\mathcal{L}}(\theta^{*};D)=\frac{1}{2}\mathbb{E}||\theta^{priv}-\theta^{*}||_2^2=\frac{1}{2}\mathbb{E}||\theta^{priv}-\frac{1}{n}\sum\limits_{i=1}^n x_i||_2^2.
	\end{equation*}
	We prove the theorem by contradiction.
		Assume Theorem \ref{th2} is false, then for any dataset $D$, there exists a $(\epsilon,\delta)$-differentially private algorithm $\mathcal{A}$, for some 
		$n=O(\frac{\sqrt{d\log(1/\delta)}}{\epsilon})$, it outputs $\theta^{priv}$ such that $\mathbb{E}[\hat{\mathcal{L}}(\theta^{priv};D)]-\hat{\mathcal{L}}(\theta^{*};D)=\frac{1}{2}\mathbb{E}||\theta^{priv}-\frac{1}{n}\sum\limits_{i=1}^n x_i||_2^2=o\left(\frac{G_{\mathcal{C}}^2\log(1/\delta)}{(\log(2d))^4\epsilon^2n^2}\right)$.
		
		In  Lemma \ref{le7}, 
  \begin{equation*}
		\begin{aligned}
		    SC_{mp}=&\min\{n:\underset{D}{\sup} (\mathbb{E}||\theta^{priv}-\bar{x}||_2^2)\leq \alpha^2\}\\
		    =& \Omega\left(\frac{G_{\mathcal{C}}\sqrt{\log(1/\delta)}}{(\log 2d)^2\alpha\epsilon}\right) \text{(Using Eq.(\ref{eq17}))}\\
		    =&o(n)~~~~~(\text{ By letting }\alpha=o\left(\frac{G_{\mathcal{C}}\sqrt{\log(1/\delta)}}{(\log(2d))^2\epsilon n}\right)),
		    \end{aligned}
		\end{equation*}
	which leads to a contradiction.
\end{proof}
\section{Omitted Proofs in Section \ref{sec:regular}}
\subsection{Proof of Theorem \ref{thm:6}}
\begin{proof}
Note that for any neighboring dataset $D$ and $D^{'}$, we have $||\nabla \hat{\mathcal{L}}(w_t,D)-\nabla \hat{\mathcal{L}}(w_t,D^{'})||_{*}\leq \frac{2L}{n}$ by the Lipschitz assumption.
Since for $\ell_p^d$-space, $||\cdot||_{*} = ||\cdot||_{\frac{p}{p-1}}$, the space $(\mathbf{E},||\cdot||_{*})$ is $\kappa$-regular with $\kappa = \min\{\frac{p}{p-1}-1, 2\ln d \}=\min\{\frac{1}{p-1},2\ln d\}$, so using the privacy guarantee provided by generalized Gaussian mechanism and the advanced composition theorem, the algorithm is $(\epsilon, \delta)$-DP.
\end{proof}

\subsection{Proof of theorem \ref{th:7}}
\begin{proof}

Observe that $\Phi(x) =\frac{\kappa}{2}||x||_{\kappa_{+}}^2$ where $\kappa=\min\{\frac{1}{p-1},2\ln d\}$ and $\kappa_{+}=\frac{\kappa}{\kappa-1}$  is 1-strongly convex w.r.t. $||\cdot||$ by the definition of $||\cdot||_{\kappa^{+}}$ and the duality between strongly convexity and smoothness.
We recall the following lemma showing that adding regularization may impair smoothness, but it also induces good properties such as relatively smooth and strongly convex.

\begin{lemma}\label{le3}(Lemma 14 in \cite{attia2022uniform})
Let $f(x)$ be a convex and $\beta$-smooth function w.r.t. $||\cdot||$ and $\Phi(x)$ be 1-strongly convex w.r.t. $||\cdot||$, then $f^{\alpha}(x) = f(x)+\alpha \cdot \Phi(x)$ for $\alpha>0$ is $(\alpha+\beta)$-smooth relative to $\Phi(x)$ as well as $\alpha$-strongly convex relative to $\Phi(x)$.
\end{lemma}

Let $w_{\alpha}^{*} = \underset{w\in\mathbf{E}}{\arg\min}\hat{\mathcal{L}}(w,D)+\alpha \Phi(w)$, $w^{*}=\underset{w\in\mathbf{E}}{\arg\min}\mathcal{L}(w)$ and $\tilde{w}^{*}=\tilde{w}^{*}(D)=\underset{w\in\mathbf{E}}{\arg\min}\hat{\mathcal{L}}(w,D)$, and $C_D=\Phi^{\frac{1}{2}}(\tilde{w}^{*})$.
Based on the optimality of $w_{\alpha}^{*}$ for the regularized objective function $\hat{\mathcal{L}}(w,D)+\alpha \Phi(w)$, along with the optimality of $\tilde{w}^{*}$ for the objective $\hat{\mathcal{L}}(w,D)$,
we have

\begin{align}
    \hat{\mathcal{L}}(w_{\alpha}^{*},D)+\alpha \Phi(w_{\alpha}^{*})&\leq \hat{\mathcal{L}}(\tilde{w}^{*},D)+\alpha \Phi(\tilde{w}^{*}),\notag\\
  \implies     \Phi(\tilde{w}^{*})-\Phi(w^{*}_{\alpha}) &\geq  \frac{
    \hat{\mathcal{L}}(w_{\alpha}^{*},D)-\hat{\mathcal{L}}(\tilde{w}^{*},D)
    }{\alpha} >0, \notag\\
  \implies  \Phi(\tilde{w}^{*})&>\Phi(w_{\alpha}^{*}). \label{eq18}
    \end{align}

Since $w_1 =0= \underset{w\in\mathbf{E}}{\arg\min}\Phi(w)$, from the first-order optimality of $w_1$, we have $\langle\nabla \Phi(w_1),w_1-w_{\alpha}^{*}\rangle \leq 0$ and thus
\begin{equation*}
\begin{aligned}
    D_{\Phi}(w_{\alpha}^{*},w_1) &=\Phi(w_{\alpha}^{*})-\Phi(w_1)-\langle\nabla \Phi(w_1), w_{\alpha}^{*}-w_1\rangle\\
    &\leq \Phi(w_{\alpha}^{*})-\Phi(w_1)\\
    &\leq \Phi(\tilde{w}^{*})-\Phi(w_1)(~\text{From Eq.( \ref{eq18})})\\
    &\leq C_D^2 ~(\text{Let}~ C_D^2 = \Phi(\tilde{w}^*)).
    \end{aligned}
\end{equation*}
Now we  rewrite our objectives in Algorithm \ref{alg3}:
\begin{equation*}
    \begin{aligned}
    &\langle \nabla \hat{\mathcal{L}}(w_t,D)+g_t, w-w_t\rangle +\beta\cdot D_{\Phi }(w,w_t)+\alpha \Phi(w)\\
    =&\langle \nabla \hat{\mathcal{L}}(w_t,D)+g_t, w-w_t\rangle +(\beta+\alpha)\cdot D_{\Phi }(w,w_t)+\alpha \Phi(x)-\alpha \cdot D_{\Phi }(w,w_t)
    \\
    =&\langle \nabla \hat{\mathcal{L}}(w_t,D)+g_t, w-w_t\rangle +(\alpha+\beta)\cdot D_{\Phi }(w,w_t) +\alpha \Phi(w) -\alpha\cdot(\Phi(w)-\Phi(w_t)-\langle\nabla \Phi(w_t),w-w_t\rangle)\\
    =&\langle \nabla \hat{\mathcal{L}}(w_t,D)+\alpha \nabla \Phi(w_t)+g_t, w-w_t\rangle +(\alpha+\beta)\cdot D_{\Phi }(w,w_t) +\alpha \Phi(w_t)\\
    =&\langle \nabla \hat{\mathcal{L}}^{(\alpha)}(w_t,D)+g_t, w-w_t\rangle +(\alpha+\beta)\cdot D_{\Phi }(w,w_t) +\alpha \Phi(w_t).
    \end{aligned}
\end{equation*}
where  $\hat{\mathcal{L}}^{(\alpha)} (w,D)\triangleq \hat{\mathcal{L}}(w,D)+\alpha \cdot\Phi(w)$ and note that $\hat{\mathcal{L}}^{(\alpha)} (w,D)$ is $(\alpha+\beta)$-smooth relative to $\Phi(x)$ as well as $\alpha$-strongly convex relative to $\Phi(w)$  according to Lemma \ref{le3}. Next, we 
recall the following ``three-point property":
\begin{lemma}\label{le4}(\textbf{Three point property}) \cite{tseng2008accelerated}.
Let $\phi(x)$ be a convex function and $D_{\Phi}(\cdot,\cdot)$ be the Bregman divergence for $\Phi(\cdot)$. For given $z$, let $z^{*}=\underset{x\in\mathbf{E}}{\arg\min}\{\phi(x)+D_{\Phi}(x,z)\}$, then for all $x\in \mathbf{E}$ we have 
\begin{equation*}
    \phi(x)+D_{\Phi}(x,z)\geq \phi(z^{*})+D_{\Phi}(z^{*},z)+D_{\Phi}(x,z^{*}).
\end{equation*}

\end{lemma}
Let $\phi(w) = \frac{1}{\alpha+\beta}\cdot\langle\nabla f(w_t)+g_t,w-w_t\rangle$ where $f(w) = \hat{\mathcal{L}}(w,D)+\alpha \cdot \Phi(w)$, set $z=w_t$ in Lemma \ref{le4}, we get
\begin{equation*}
\frac{1}{\alpha+\beta}\cdot\langle\nabla f(w_t)+g_t,w-w_t\rangle +D_{\Phi}(w,w_t)\geq \frac{1}{\alpha+\beta}\cdot\langle\nabla f(w_t)+g_t,w_{t+1}-w_t\rangle+D_{\Phi}(w_{t+1},w_t)+D_{\Phi}(w,w_{t+1}), 
\end{equation*}
which implies 
\begin{equation*}
(\alpha+\beta)\cdot D_{\Phi}(w_{t+1},w_t) \leq \langle \nabla f(w_t)+g_t,w-w_{t+1}\rangle +(\alpha+\beta)\cdot (D_{\Phi}(w,w_t)-D_{\Phi}(w,w_{t+1})). 
\end{equation*}
Since $f(w)$ is $(\alpha+\beta)$-smooth relative to $\Phi(w)$, we have
\begin{equation}\label{eq2}
\begin{aligned}
f(w_{t+1})\leq &f(w_t)+\langle \nabla f(w_t), w_{t+1}-w_t\rangle +(\alpha+\beta)\cdot D_{\Phi}(w_{t+1},w_t)\\
\leq & f(w_t)+\langle \nabla f(w_t), w-w_t\rangle +(\alpha+\beta)\cdot (D_{\Phi}(w,w_t)-D_{\Phi}(w,w_{t+1}))+\langle g_t, w-w_{t+1}\rangle.
\end{aligned}
\end{equation}
Since $f(w)$ is $\alpha$-strongly convex relative to $\Phi(w)$, from the definition, we have
\begin{equation*}
f(w_t)+ \langle \nabla f(w_t), w-w_t\rangle \leq f(w) -\alpha \cdot D_{\Phi}(w,w_t).
\end{equation*}
So inequality (\ref{eq2}) becomes 
\begin{equation}\label{eq3}
\begin{aligned}
f(w_{t+1})&\leq f(w) -\alpha \cdot D_{\Phi}(w,w_t) +(\alpha+\beta)\cdot (D_{\Phi}(w,w_t)-D_{\Phi}(w,w_{t+1}))+\langle g_t, w-w_{t+1}\rangle\\
&\leq f(w) +\beta\cdot D_{\Phi}(w,w_t)-(\alpha+\beta)\cdot D_{\Phi}(w,w_{t+1})+\langle g_t, w-w_{t+1}\rangle.
\end{aligned}
\end{equation}
Note that  for any constant $a>0$
\begin{equation*}
\begin{aligned}
\langle g_t, w-w_{t+1}\rangle \leq & a\cdot ||g_t||_{*}^2 +\frac{1}{2a}\cdot ||w-w_{t+1}||^2\\
\leq & a\cdot ||g_t||_{*}^2 +\frac{1}{2a} \cdot D_{\Phi}(w, w_{t+1}),
\end{aligned}
\end{equation*}
where the last inequality is due to $\Phi$ being 1-strongly convex w.r.t. $\|\cdot\|$.  Now inequality (\ref{eq3}) can be written as
\begin{equation}\label{eq4}
f(w_{t+1})\leq f(w) +\beta\cdot D_{\Phi}(w,w_t)-(\alpha+\beta-\frac{1}{2a})\cdot D_{\Phi}(w,w_{t+1})+a\cdot ||g_t||_{*}^2.
\end{equation}
Let $w$ in Eq. (\ref{eq4}) to be $w_{\alpha}^* = \arg\min f(w)$, let $a = \frac{1}{\alpha}$, we have

\begin{equation*}
\begin{aligned}
D_{\Phi}(w_{\alpha}^{*}, w_{t+1})\leq &\frac{\beta}{\alpha+\beta-\frac{1}{2a}} \cdot D_{\Phi}(w_{\alpha}^{*},w_t) +O\left(\frac{a}{\alpha+\beta -\frac{1}{2a}} \cdot \|g_t||_{*}^2\right)\\
\leq & \frac{1}{1+\frac{\alpha}{2\beta}} \cdot D_{\Phi}(w_{\alpha}^{*}, w_t) +O\left(\frac{1}{\alpha\beta}\cdot ||g_t||_{*}^2\right).
\end{aligned}
\end{equation*}
Letting $t = 1,2,\cdots, T$, add these inequalities together, we have
\begin{equation*}
\begin{aligned}
\mathbb{E}[D_{\Phi}(w_{\alpha}^{*},w_{T+1})]\leq &\left(\frac{1}{1+\frac{\alpha}{2\beta}}\right)^T \cdot D_{\Phi}(w_{\alpha}^{*},w_1)+O\left(\frac{1}{\alpha^2}\cdot g^2\right)\\
= & \left(1+\frac{\alpha}{2\beta}\right)^{-T} \cdot D_{\Phi}(w_{\alpha}^{*},w_1)+O\left(\frac{1}{\alpha^2}\cdot g^2\right)\\
\leq &2^{-\frac{\alpha T}{2\beta}}\cdot D_{\Phi}(w_{\alpha}^{*},w_1)+O\left(\frac{1}{\alpha^2}\cdot g^2\right)\\
\leq &2^{-\frac{\alpha T}{2\beta}}\cdot C_D^2+O\left(\frac{1}{\alpha^2}\cdot g^2\right),
\end{aligned}
\end{equation*}
where the expectation is taken over all $g_1, \cdots, g_T$ and $g^2 = \mathbb{E}[||g_t||_{*}^2]$. The last inequality utilizes the fact that $(1+\frac{1}{x})^x\geq 2$ for all $x\geq 1$ and note that $\frac{2\beta}{\alpha}\geq 1$. 
Since $\Phi$ is strongly convex, we also have 
\begin{equation*}
\frac{1}{2}\mathbb{E}[||w_{\alpha}^{*}-w_{T+1}||^2]\leq \mathbb{E}[D_{\Phi}(w_{\alpha}^{*},w_{T+1})] \leq 2^{-\frac{\alpha T}{2\beta}}\cdot C_D^2+O\left(\frac{1}{\alpha^2}\cdot g^2\right).
\end{equation*}
Thus, we have
\begin{equation*}
\mathbb{E}[||w_{\alpha}^{*}-w_{T+1}||] \leq O\left( 2^{-\frac{\alpha T}{4\beta}}\cdot C_D +\frac{1}{\alpha}\cdot g \right).
\end{equation*}
Now we consider a neighboring data $D'$ of $D$ where they differ by the $i$-th entry. Denote $w_{\alpha}^{*'}=\hat{\mathcal{L}}(w,D')+\alpha \cdot \Phi(w)$ and $w^{'}_{T+1}$ as the parameters of the algorithm on $D'$. Then, similar to the previous case we can get 
\begin{equation*}
\mathbb{E}[||w_{\alpha}^{*'}-w^{'}_{T+1}||] \leq O\left( 2^{-\frac{\alpha T}{4\beta}}\cdot C_D +\frac{1}{\alpha}\cdot g \right).
\end{equation*}
Next, we will bound the term $||w_{\alpha}^{*}-w_{\alpha}^{*'}||$ by the following lemma.
\begin{lemma}\label{le5}
Let $f_1,f_2:\mathbf{E}\rightarrow \mathbb{R}$ be convex and $\alpha$-strongly convex (relatively). Let $x_1=\underset{x\in\mathbf{E}}{\arg\min} f_1(x)$ and $x_2=\underset{x\in\mathbf{E}}{\arg\min } f_2(x)$, then
\begin{equation*}
||x_2-x_1||\leq \frac{2}{\alpha}||\nabla (f_2-f_1)({x_1})||_{*}.
\end{equation*}
\end{lemma}
From the above lemma, let $f_1(w) = \hat{\mathcal{L}}(w,D)+\alpha \cdot \Phi(w)$ and $f_2(w) = \hat{\mathcal{L}}(w,D^{'})+\alpha \cdot \Phi(w)$, we can get 
\begin{equation*}
||w_{\alpha}^{*}-w_{\alpha}^{*'}||\leq \frac{2||{\nabla \ell(w_{\alpha}^{*};x_i)-\nabla \ell(w_{\alpha}^{*};x_i^{'})}||_{*}}{n\alpha}\leq \frac{4L}{n\alpha}.
\end{equation*}

In total
\begin{equation*}
\begin{aligned}
\mathbb{E}[||w_{T+1}^{'}-w_{T+1}||]\leq & O\left(  2^{-\frac{\alpha T}{4\beta}} \cdot C_D +\frac{L}{n\alpha}+\frac{g}{\alpha}\right)\\
=& O\left(  2^{-\frac{\alpha T}{4\beta}} \cdot C_D +\frac{L}{n\alpha}+\frac{L\sqrt{\log(1/\delta) d \kappa T}}{\alpha n \epsilon}\right).
\end{aligned}
\end{equation*}
Similarly, we can also show that for any $t$ we have 
\begin{equation*}
\begin{aligned}
\mathbb{E}[||w_{t+1}^{'}-w_{t+1}||]\leq & O\left(  2^{-\frac{\alpha t}{4\beta}} \cdot C_D +\frac{L}{n\alpha}+\frac{g}{\alpha}\right)\\
=& O\left(  2^{-\frac{\alpha t}{4\beta}} \cdot C_D +\frac{L}{n\alpha}+\frac{L\sqrt{\log(1/\delta) d \kappa T}}{\alpha n \epsilon}\right).
\end{aligned}
\end{equation*}

Now we go back to Eq. (\ref{eq4}),
\begin{equation*}
\begin{aligned}
f(w_{t+1})-f(w_{\alpha}^{*})\leq & \beta\cdot D_{\Phi}(w_{\alpha}^{*},w_t)-(\alpha+\beta-\frac{1}{2a})\cdot D_{\Phi}(w_{\alpha}^{*},w_{t+1})+a\cdot ||g_t||_{*}^2\\
\leq  & \beta\cdot D_{\Phi}(w_{\alpha}^{*},w_t)-(\beta+\frac{\alpha}{2})\cdot D_{\Phi}(w_{\alpha}^{*},w_{t+1})+O\left(\frac{1}{\alpha}\cdot ||g_t||_{*}^2\right).
\end{aligned}
\end{equation*}
Since 
\begin{equation*}
\begin{aligned}
&\sum_{t=1}^T \left(\frac{2\beta+\alpha}{2\beta}\right)^t \cdot \mathbb{E}[f(w_{t+1})-f(w_{\alpha}^{*})]\\
\leq &\beta \left[\sum_{t=1}^T \left( \frac{2\beta+\alpha}{2\beta}\right)^t   \cdot D_{\Phi}(w_{\alpha}^{*},w_t)- \sum_{t=1}^T \left( \frac{2\beta+\alpha}{2\beta}\right)^{t +1}   \cdot D_{\Phi}(w_{\alpha}^{*},w_{t+1})\right]+ O\left( \sum_{t=1}^T\left(\frac{2\beta+\alpha}{2\beta}\right)^{t}\cdot \frac{1}{\alpha} g^2   \right)\\
= & \beta\left[\frac{2\beta+\alpha}{2\beta}\cdot D_{\Phi}(w_{\alpha}^{*},w_1) -\left(\frac{2\beta+\alpha}{2\beta}\right)^{T+1} \cdot D_{\Phi}(w_{\alpha}^{*},w_{T+1})\right]+ O\left( \sum_{t=1}^T\left(\frac{2\beta+\alpha}{2\beta}\right)^{t}\cdot \frac{1}{\alpha} g^2   \right)\\
\leq & \frac{2\beta+\alpha}{2}\cdot D_{\Phi}(w_{\alpha}^{*},w_1) + O\left( \sum_{t=1}^T\left(\frac{2\beta+\alpha}{2\beta}\right)^{t}\cdot \frac{1}{\alpha} g^2   \right).
\end{aligned}
\end{equation*}
Let 
\begin{equation*}
\hat{w}=\frac{\sum_{t=1}^T \left(\frac{2\beta+\alpha}{2\beta}\right)^t  \cdot w_{t+1}}{\sum_{t=1}^T \left(\frac{2\beta+\alpha}{2\beta}\right)^t }.
\end{equation*}
And we have

\begin{align}
\mathbb{E}[f(\hat{w})-f(w_{\alpha}^{*})] &= \mathbb{E}\left[f\left(\frac{\sum_{t=1}^T \left(\frac{2\beta+\alpha}{2\beta}\right)^t  \cdot w_{t+1}}{\sum_{t=1}^T \left(\frac{2\beta+\alpha}{2\beta}\right)^t }\right)-f(w_{\alpha}^{*})\right]\notag\\
& \leq \mathbb{E} \left[ \frac{\sum_{t=1}^T \left(\frac{2\beta+\alpha}{2\beta}\right)^t \cdot f(w_{t+1})}{\sum_{t=1}^T \left(\frac{2\beta+\alpha}{2\beta}\right)^t  }   -f(w_{\alpha}^{*})\right]\notag\\
& =\frac{\mathbb{E}\left[ \sum_{t=1}^T \left(\frac{2\beta+\alpha}{2\beta}\right)^t \cdot\left(f(w_{t+1})-f(w^{*}_{\alpha})\right)\right]}{\sum_{t=1}^T \left(\frac{2\beta+\alpha}{2\beta}\right)^t }\notag\\
& = \frac{\sum_{t=1}^T\left( \frac{2\beta +\alpha}{2\beta}  \right)^t \cdot \mathbb{E}[f(w_{t+1})-f(w^{*}_{\alpha})]}{\sum_{t=1}^T \left(\frac{2\beta+\alpha}{2\beta}\right)^t }\notag\\
& \leq \frac{(2\beta+\alpha)\cdot D_{\Phi}(w_{\alpha}^{*},w_1)}{2\cdot\sum_{t=1}^T \left(\frac{2\beta+\alpha}{2\beta}\right)^t } +O\left( \frac{1}{\alpha} g^2 \right)\notag\\
&=\frac{\alpha \cdot D_{\Phi}(w_{\alpha}^{*},w_1)}{2\left[\left(\frac{2\beta+\alpha}{2\beta} \right)^T-1\right]} +O\left( \frac{1}{\alpha} g^2\right)\notag\\
& \leq \frac{\alpha}{2} \cdot D_{\Phi}(w_{\alpha}^{*},w_1)+O\left( \frac{1}{\alpha} g^2 \right)\label{eq5}\\
& \leq O\left(\alpha \cdot D_{\Phi}(w_{\alpha}^{*},w_1)+\frac{1}{\alpha} g^2\right),\notag
\end{align}

where we used the fact that when $T\geq \frac{2\beta}{\alpha}$,
\begin{equation*}
\left(\frac{2\beta+\alpha}{2\beta} \right)^T=
(1+\frac{\alpha}{2\beta})^T
\geq 2
\end{equation*}
in inequality (\ref{eq5}).

Denote $\tilde{w}^{*}=\underset{w\in \mathbf{E}}{\arg\min} \hat{\mathcal{L}}(w,D)$, we have
\begin{equation*}
\begin{aligned}
\mathbb{E}[\hat{\mathcal{L}}(\hat{w},D) -\hat{\mathcal{L}}(\tilde{w}^{*},D)] &= \mathbb{E}[\hat{\mathcal{L}}^{(\alpha)}(\hat{w},D)-\hat{\mathcal{L}}^{(\alpha)}(\tilde{w}^{*},D)]+\alpha \cdot \Phi(\tilde{w}^{*})-\alpha \cdot \Phi(\hat{w})\\
& \leq \mathbb{E}[\hat{\mathcal{L}}^{(\alpha)}(\hat{w},D) -\hat{\mathcal{L}}^{(\alpha)}(w_{\alpha}^{*},D)]+\alpha \cdot \Phi(\tilde{w}^{*})-\alpha \cdot \Phi(\hat{w})\\
&\leq  O \left(\alpha \cdot D_{\Phi}(w_{\alpha}^{*},w_1)\right)+O\left( \frac{1}{\alpha} g^2 \right)+ \alpha \cdot \Phi(\tilde{w}^{*})-\alpha \cdot \Phi(\hat{w})\\
&\leq  O \left(\alpha \cdot D_{\Phi}(\tilde{w}^{*},w_1)\right)+O\left( \frac{1}{\alpha} g^2 \right)+\alpha \cdot C_D^2\\
&\leq  O(\alpha\cdot C_D^2+ \frac{1}{\alpha} g^2 ).
\end{aligned}
\end{equation*}
Now we bound the sensitivity of $\hat{w}$:
\begin{equation}\label{eq7}
\begin{aligned}
\mathbb{E}[||\hat{w}-\hat{w}^{'}||]&\leq \frac{\sum_{t=1}^T \left(\frac{2\beta+\alpha}{2\beta}\right)^t \mathbb{E}[||w_{t+1}-w_{t+1}^{'}||]}{\sum_{t=1}^T \left(\frac{2\beta+\alpha}{2\beta}\right)^t}\\
& \leq O\left( \frac{\sum_{t=1}^T \left(\frac{2\beta+\alpha}{2\beta}\right)^t2^{-\frac{\alpha t}{4\beta}} \cdot C_D}{\sum_{t=1}^T \left(\frac{2\beta+\alpha}{2\beta}\right)^t}  +\frac{L}{n\alpha}+\frac{L\sqrt{\log(1/\delta) d \kappa T}}{\alpha n \epsilon}\right).
\end{aligned}
\end{equation}
We bound the first term above:
\begin{equation} \label{eq6}
\begin{aligned}
\frac{\sum_{t=1}^T \left(\frac{2\beta+\alpha}{2\beta}\right)^t2^{-\frac{\alpha t}{4\beta}} \cdot C_D}{\sum_{t=1}^T \left(\frac{2\beta+\alpha}{2\beta}\right)^t} = &\frac{ C_D \cdot\sum_{t=1}^T \left[\frac{2\beta+\alpha}{2\beta} \cdot \left(\frac{1}{2}\right)^{\frac{\alpha }{4\beta}}\right]^t }{\sum_{t=1}^T \left(\frac{2\beta+\alpha}{2\beta}\right)^t} \\
=& C_D \cdot \frac{1-\frac{2\beta+\alpha}{2\beta}}{\frac{2\beta+\alpha}{2\beta}\cdot \left[1-\left(\frac{2\beta+\alpha}{2\beta}\right)^T\right]} \cdot \frac{\frac{2\beta+\alpha}{2\beta}\cdot \left(\frac{1}{2}\right)^{\frac{\alpha}{4\beta}} \cdot \left( 1-\left[\frac{2\beta+\alpha}{2\beta}\cdot \left(\frac{1}{2}\right)^{\frac{\alpha}{4\beta}} 
\right]^T
\right)  }{1-\frac{2\beta+\alpha}{2\beta}\cdot \left(\frac{1}{2}\right)^{\frac{\alpha}{4\beta}}}\\
=& C_D\cdot  \left(\frac{1}{2}\right)^{\frac{\alpha}{4\beta}} \cdot\frac{\alpha}{(2\beta+\alpha)\cdot \left(\frac{1}{2}\right)^{\frac{\alpha}{4\beta}}-2\beta} \cdot \frac{\left[\frac{2\beta+\alpha}{2\beta}\cdot \left(\frac{1}{2}\right)^{\frac{\alpha}{4\beta}} 
\right]^T-1}{\left(\frac{2\beta+\alpha}{2\beta}\right)^T-1}.
\end{aligned}
\end{equation}
Consider function $f(x)=(1+x)\cdot a^x$. Its derivative $f'(x)=\ln a \cdot a^x+a^x+\ln a\cdot x\cdot a^x=a^x(\ln a+1+\ln a \cdot x)$, let $a=\frac{1}{\sqrt{2}}$, then $f'(x)>0$ for $x\in [0,1]$. Thus we have $(1+x)\cdot (\frac{1}{\sqrt{2}})^x>1$.
Let $x=\frac{\alpha}{2\beta}$, we have $(1+\frac{\alpha}{2\beta})\cdot (\frac{1}{2})^{\frac{\alpha}{4\beta}}>1$, namely $(2\beta+\alpha)\cdot (\frac{1}{2})^{\frac{\alpha}{4\beta}}-2\beta>0$.

In the following, we bound the term $\frac{\alpha}{(2\beta+\alpha)\cdot \left(\frac{1}{2}\right)^{\frac{\alpha}{4\beta}}-2\beta}$.
\begin{equation*}
\begin{aligned}
\frac{\alpha}{(2\beta+\alpha)\cdot \left(\frac{1}{2}\right)^{\frac{\alpha}{4\beta}}-2\beta} &=\frac{\alpha}{(2\beta+\alpha)\cdot \left((\frac{1}{2})^{\frac{\alpha}{4\beta}}-1\right)+\alpha}\\
&\leq \frac{\alpha}{(2\beta+\alpha)\cdot(-\frac{\alpha}{4\beta})+\alpha}\\
&=\frac{1}{\frac{1}{2}-\frac{\alpha}{4\beta}}\leq 4 ~(\text{Assume} ~\frac{\alpha}{\beta}\leq 1),
\end{aligned}
\end{equation*}
where we use the fact that $(\frac{1}{2})^{\frac{\alpha}{4\beta}}-1\geq -\frac{\alpha}{4\beta}$. (To prove this is to prove that $2^{\frac{\alpha}{4\beta}}(1-\frac{\alpha}{4\beta})\leq 1$. Let $f(x)=a^x(1-x)$. The derivative $f'(x)=\ln a \cdot a^x-\ln a\cdot x\cdot a^x-a^x=a^x\cdot(\ln a-x \cdot \ln a -1)<0$ when $a<e$. So $f(x)$ decreases in $[0,1]$, and thus $f(x)\leq 1$, $\forall x\in [0,1]$. Let $a=2$  and $x=\frac{\alpha}{4\beta}$, and we will get $2^{\frac{\alpha}{4\beta}}\cdot (1-\frac{\alpha}{4\beta})\leq 1$.)

Now we bound the term $\frac{\left[\frac{2\beta+\alpha}{2\beta}\cdot \left(\frac{1}{2}\right)^{\frac{\alpha}{4\beta}} 
\right]^T-1}{\left(\frac{2\beta+\alpha}{2\beta}\right)^T-1}$.
\begin{equation*}
\begin{aligned}
\frac{\left[\frac{2\beta+\alpha}{2\beta}\cdot \left(\frac{1}{2}\right)^{\frac{\alpha}{4\beta}} 
\right]^T-1}{\left(\frac{2\beta+\alpha}{2\beta}\right)^T-1}=&\frac{\left(\frac{2\beta+\alpha}{2\beta}\right)^T\cdot (\frac{1}{2})^{\frac{\alpha T}{4\beta}}-(\frac{1}{2})^{\frac{\alpha T}{4\beta}}+(\frac{1}{2})^{\frac{\alpha T}{4\beta}}-1
}{\left(\frac{2\beta+\alpha}{2\beta}\right)^T-1}\\
=& \left(\frac{1}{2}\right)^{\frac{\alpha T}{4\beta}} +\frac{(\frac{1}{2})^{\frac{\alpha T}{4\beta}}-1}{\left(\frac{2\beta+\alpha}{2\beta}\right)^T-1}\\
<&\left(\frac{1}{2}\right)^{\frac{\alpha T}{4\beta}}.
\end{aligned}
\end{equation*}
Thus, Eq. (\ref{eq6}) becomes
\begin{equation*}
\frac{\sum_{t=1}^T \left(\frac{2\beta+\alpha}{2\beta}\right)^t2^{-\frac{\alpha t}{4\beta}} \cdot C_D}{\sum_{t=1}^T \left(\frac{2\beta+\alpha}{2\beta}\right)^t} = O\left(C_D \cdot \left(\frac{1}{2}\right)^{\frac{\alpha (T+1)}{4\beta}}\right).
\end{equation*}
Bring this back to Eq.(\ref{eq7}) and we can get
\begin{equation*}
\mathbb{E}[||\hat{w}-\hat{w}'||]
\leq O\left( C_D \cdot2^{\frac{-\alpha (T+1)}{4\beta}} +\frac{L}{n\alpha}+\frac{L\sqrt{\log(1/\delta) d \kappa T}}{\alpha n \epsilon}\right).
\end{equation*}
Since the loss is $L$-Lipschitz w.r.t $\|\cdot\|$, we can see the generalization error $\mathbb{E}[\mathcal{L}(\hat{w})-\hat{\mathcal{L}}(\hat{w},D)] \leq L \cdot O\left( C_D \cdot2^{\frac{-\alpha (T+1)}{4\beta}} +\frac{L}{n\alpha}+\frac{L\sqrt{\log(1/\delta) d \kappa T}}{\alpha n \epsilon}\right).$

Take $\alpha=\frac{4\beta}{T+1}\log_2 \frac{n}{T}$,

\begin{equation*}
\begin{aligned}
\mathbb{E}[\mathcal{L}(\hat{w})]-\mathcal{L}(w^{*})&=\mathbb{E}[\mathcal{L}(\hat{w})-\hat{\mathcal{L}}(\hat{w},D)] + \mathbb{E}[\hat{\mathcal{L}}(\hat{w},D)-\hat{\mathcal{L}}(w^{*},D)]\\
&\leq L\cdot \mathbb{E}[||\hat{w}-\hat{w}'||]+\mathbb{E}[\hat{\mathcal{L}}(\hat{w},D)-\hat{\mathcal{L}}(\tilde{w}^{*},D)]\\
&=O\left( L\cdot 2^{\frac{-\alpha(T+1)}{4\beta}}\cdot \mathbb{E}[C_D]+\frac{L^2}{n\alpha}+\frac{L^2\sqrt{\log(1/\delta)d \kappa T}}{\alpha n \epsilon}
+\alpha\cdot \mathbb{E}[C_D^2]+\frac{1}{\alpha}\cdot \frac{L^2 \log(1/\delta)d \kappa T}{n^2\epsilon^2}
\right)\\
&=\tilde{O}\left(\frac{T\sqrt{\kappa}}{n}+\frac{T^{\frac{3}{2}}\sqrt{d\log(1/\delta)\kappa}}{n\epsilon}+\frac{T^2 d\log (1/\delta)\kappa}{n^2\epsilon^2}+\frac{\kappa }{T}
\right) ~(\text{By substituting} ~\alpha=\frac{4\beta}{T+1}\log_2 \frac{n}{T})\\
&=\tilde{O}\left(\frac{T\sqrt{\kappa}}{n}+\frac{T^{\frac{3}{2}}\sqrt{d\log(1/\delta)\kappa}}{n\epsilon}+\frac{\kappa}{T} 
\right)\\
&\leq\tilde{O}\left(\frac{T^{\frac{3}{2}}\sqrt{d\log(1/\delta)\kappa}}{n\epsilon}+\frac{\kappa}{T} 
\right) ~(\text{Since $T=O\left(\sqrt{n\sqrt{\kappa}}\right)$})\\
&=  \tilde{O}\left( \kappa^ \frac{4}{5}\left(\frac{\sqrt{d\log(1/\delta) }}{n\epsilon}\right)^{\frac{2}{5}}
\right)~(\text{By letting }~T=\Theta\left(\left(\frac{n\epsilon\sqrt{k}}{\sqrt{d\log(1/\delta)}}\right)^{\frac{2}{5}}\right)),
\end{aligned}
\end{equation*}
where $\tilde{O}$ hides a factor of $\mathbb{E}[\tilde{C}_D^2]$ with $\tilde{C}_D^2 = \|\tilde{w}^*\|_{\kappa_+}^2$ and $\tilde{w}^{*}=\underset{w\in\mathbf{E}}{\arg\min} \hat{\mathcal{L}}(w,D)$. 

(Note that since we assume $n=O\left( \frac{\epsilon^4}{(d \log (1/\delta))^2\kappa^{1/2}}\right)$,  the constraint $T=O\left(\sqrt{n\sqrt{\kappa}}\right)$ comes for free when  letting $T=\Theta\left(\left(\frac{n\epsilon\sqrt{k}}{\sqrt{d\log(1/\delta)}}\right)^{\frac{2}{5}}\right)$).
\subsection{Proof of Theorem \ref{thm:8}}
To be self-contained, we first review the Phased DP-SGD algorithm in \cite{feldman2020private}. Since we are concerned about the unconstrained case, we slightly modify the original Phased DP-SGD algorithm by  eliminating the projection step.
	\begin{algorithm}
	\caption{Phased-DP-SGD algorithm \cite{feldman2020private} \label{alg:6}}
	\begin{algorithmic}[1]
		\State {\bfseries Input:}  	Dataset $S=\{x_1,\cdots,x_n\} $, convex loss $\ell$, step size $\eta$ (will be specified later), privacy parameter $\epsilon$ and (or) $\delta$.
		\State 	Set $k= \lceil \log_2 n\rceil$. Partite the whole dataset $S$ into $k$ subsets $\{S_1,\cdots,S_k\}$. Denote $n_i$ as the number of samples in $S_i$, {\em i.e.,} $|S_i|=n_i$, where $n_i=\lfloor 2^{-i}n \rfloor$. Moreover, set $w_0=0$. 
		\For {$i=1,\cdots ,k$}
		\State  Let $\eta_i=4^{-i}\eta$, $w_i^1=w_{i-1}$. 
		\For{$t= 1,\cdots,n_i$}
		\State Update $w_i^{t+1}=w_i^{t}-\eta_i\nabla \ell(w_{i}^{t},x_i^t)$, where $x_i^t$ is the $t$-th sample of the set $S_i$. 
		
		\EndFor
		\State Set $\overline{w}_i=\frac{1}{n_i+1}\sum \limits_{t=1}^{n_i+1} w_i^t$. 
		\State For $(\epsilon,\delta)$-DP, $w_i=\overline{w}_i+\xi_i$, where $\xi_i \sim \mathcal{N}(0,\sigma_i^2\mathbb{I}_d)$ with $\sigma_i=\frac{4L\eta_i\sqrt{\log (1/\delta)}}{\epsilon}$. 
		\EndFor \\
		\Return $w_k$
		\end{algorithmic}
\end{algorithm}
		\begin{lemma}\label{alemma:14}(Modification of Theorem 4.4 in \cite{feldman2020private})
		Let  $\ell(\cdot, x)$ be $\beta$-smooth, convex and $L$-Lipschitz function over $\mathbb{R}^d$ for each $x$. If we set $\eta=\frac{1}{L}\min\{\frac{4}{\sqrt{n}}, \frac{\epsilon}{2\sqrt{d\log(1/\delta)}}\}$  and if $\eta\leq \frac{1}{\beta}$ ({\em i.e.,} $n$ is sufficiently large), then Algorithm \ref{alg:6} will be $(\epsilon,\delta)$-DP for all $\epsilon\leq 2\log (1/\delta)$. The output satisfies
		\begin{equation*}
	\mathbb{E}[\mathcal{L}(w_k)]-\mathcal{L}(\theta^*) \leq O\left(L\|\theta^*\|_2^2\left(\frac{1}{\sqrt{n}}+\frac{\sqrt{d\log(1/\delta)}}{\epsilon n}\right)\right). 
		\end{equation*}
	\end{lemma} 
\begin{proof}
First, we have the following result, which can be found in the  standard convergence bounds for SGD
    \begin{lemma}\label{alemma:3}
Consider the Gradient Descent method with initial parameter $w_0$, fixed stepsize $\eta$ and iteration number $T$, assume in the $t$-the iteration we have $w_t$, then for any $w$ we have 
\begin{equation}
   \mathcal{L}(\bar{w}_T, D)-\mathcal{L}(w, D)\leq O(\frac{\|w_0-w\|_2^2}{\eta T}+\eta L^2),
\end{equation}
where $\bar{w}_T=\frac{w_0+ w_1+w_2+\cdots+w_T}{T+1}$. 
\end{lemma}

Now we focus on the $i$-th epoch, by Lemma \ref{alemma:3} we have for any $w$
\begin{equation}\label{eq:a.3}
     \mathbb{E} [\mathcal{L}(\bar{w}_i)]-\mathcal{L}(w)\leq O(\frac{\mathbb{E}[\|w_{i-1}-w\|_2^2]}{\eta T}+\eta L^2).
\end{equation}
Now let's be back to our proof. We have (denote $\theta^*=\arg\min_{w\in \mathbb{R}^d } \mathcal{L}(w)$ )
\begin{align*}
      & \mathcal{L}(w_k)-\mathcal{L}(\theta^*)=  \underbrace{ \mathcal{L}(w_k)-\mathcal{L}(\bar{w}_k)}_{A}
      + \underbrace{\sum_{i=2}^k (\mathcal{L}(\bar{w}_i)-\mathcal{L}(\bar{w}_{i-1}))}_{B}+\underbrace{\mathcal{L}(\bar{w}_1)-\mathcal{L}(\theta^*)}_{C}
\end{align*}
For term $A$, by the Lipschitz property we have 
\begin{align*}
    \mathbb{E} [\mathcal{L}(w_k)]-\mathcal{L}(\bar{w}_k) &\leq L \mathbb{E}[\|w_k-\bar{w}_k\|_2]\leq L\mathbb{E}\|\zeta_k\|_2. 
\end{align*}
For each term of $B$ by (\ref{eq:a.3}) and take $w=\bar{w}_{i-1}$ we have 
\begin{align}\label{eq:a.4}
   \mathbb{E} [\mathcal{L}(\bar{w}_i)]-\mathcal{L}(\bar{w}_{i-1}) &\leq O(\frac{\mathbb{E}[\|w_{i-1}- \bar{w}_{i-1}\|^2_2]}{\eta_i n_i}+\eta_i L^2) = O(\frac{\mathbb{E} [\|\zeta_i\|^2_2]}{\eta_i n_i}+\eta_i L^2)
\end{align}
For term $C,$ by (\ref{eq:a.3}) and take $w=\theta^*$ we have 
\begin{align}\label{eq:a.5}
    \mathbb{E}[\mathcal{L}(\bar{w}_1)]-L(\theta^*)& \leq O(\frac{\|\theta^*\|^2_2}{\eta_1 n_1}+\eta_1 L^2).
\end{align}
Thus, combing  (\ref{eq:a.3}), (\ref{eq:a.4}) and (\ref{eq:a.5}), we have 
\begin{align}\label{eq:a.6}
   \mathbb{E} [\mathcal{L}(w_k)]-\mathcal{L}(\theta^*)\leq  O(L\mathbb{E}[\|\zeta_k\|_2]+ \frac{\|\theta^*\|^2_2}{\eta_1 n_1}+\eta_1 L^2 +\sum_{i=2}^k (\frac{\mathbb{E} [\|\zeta_i\|^2_2]}{\eta_i n_i}+\eta_i L^2)
\end{align}
Now, we analyze the case of $(\epsilon, \delta)$-DP, it is almost the same for $\epsilon$-DP. Specifically, we have $\mathbb{E}[\|\zeta_i\|_2^2]=O(\frac{dL^2\eta_i^2\log (1/\delta)}{\epsilon^2})$. Thus, 
\begin{align*}
   L\mathbb{E}[\|\zeta_k\|_2]\leq 
   L\sqrt{ \mathbb{E}\|\zeta_k\|_2^2}&= L^2\cdot \frac{\sqrt{d\log(1/\delta)}\eta_k}{\epsilon}\\
   &=
   O(\frac{\sqrt{d\log (1/\delta)}\eta L^2}{n^2\epsilon})\\
   &=O(L(\frac{\sqrt{d\log (1/\delta)}}{n^{2.5}\epsilon}+\frac{1}{n^2}))  . 
\end{align*}
where the second inequality is due to $\eta= \frac{1}{L}\min \{\frac{1}{\sqrt{n}}, \frac{\epsilon}{\sqrt{d\log (1/\delta)}}\}$. And 
\begin{align*}
    \frac{\|\theta^{*}\|^2_2}{\eta_1 n_1}+\eta_1 L^2 &=O( \frac{\|\theta^*\|^2_2}{\eta n}+ \eta L^2) \\
    &=O(\|\theta^*\|^2_2 L(\frac{1}{n}\max \{\sqrt{n}, \frac{\sqrt{d\log (1/\delta)}}{\epsilon}\}+\frac{1}{\sqrt{n}}))\\
    & \leq O(\|\theta^*\|^2_2 L(\frac{1}{\sqrt{n}}+ \frac{\sqrt{d\log (1/\delta)}}{n\epsilon})),
\end{align*}
where the second inequality is due to $\eta= \frac{1}{L}\min \{\frac{1}{\sqrt{n}}, \frac{\epsilon}{\sqrt{d\log (1/\delta)}}\}$.
\begin{align*}
    \sum_{i=2}^k (\frac{\mathbb{E} \|\zeta_i\|^2_2}{\eta_i n_i}+\eta_i L^2) &=O(  \sum_{i=2}^k (\frac{dL^2\eta_i^2\log (1/\delta)}{\eta_i n_i \epsilon^2}+\eta_iL^2) \\
    &=O(  \sum_{i=2}^k \frac{ 2^{-i}}{n\eta }+4^{-i}\frac{L}{\sqrt{n}} )\\
    &= O(  \sum_{i=2}^k (2^{-i} (\frac{1}{n\eta}+ \frac{L}{\sqrt{n}}))\\
    &\leq O(  \sum_{i=2}^\infty (2^{-i}L (\frac{1}{n}\max \{\sqrt{n}, \frac{\sqrt{d\log (1/\delta)}}{\epsilon}\}+ \frac{1}{\sqrt{n}}))\\
    &\leq O(L(\frac{1}{\sqrt{n}}+ \frac{\sqrt{d\log (1/\delta)}}{n\epsilon})).
\end{align*}
Thus, combining with the previous three bounds into (\ref{eq:a.6}), we have our result. 

\end{proof}
Next, we will prove  Theorem \ref{thm:8} via Lemma \ref{alemma:14}. Specifically, we have the following result. 
\begin{theorem}
    For the $\ell_p^d$ space with $1<p<2$ and  suppose Assumption \ref{as3} holds. Then Algorithm \ref{alg:6} will be $(\epsilon,\delta)$-DP for all $\epsilon\leq 2\log (1/\delta)$. If we set $\eta=\frac{1}{L}\min\{\frac{4}{\sqrt{n}}, \frac{\epsilon}{2\sqrt{d\log(1/\delta)}}\}$, the output satisfies 
\begin{equation}
   \mathbb{E}[\mathcal{L}(\hat{w})]-\mathcal{L}(\theta^*)\leq  O\left(Ld^{1-\frac{2}{p}}\|\theta^*\|^2\left(\frac{1}{\sqrt{n}}+\frac{\sqrt{d\log(1/\delta)}}{\epsilon n}\right)\right). 
\end{equation}
\end{theorem}
\begin{proof}
We bound the $\|\cdot\|_2$-diameter and Lipschitz constant for the $\ell_p^d$-setting. First we have that $\|\theta^*\|_2\leq d^{\frac{1}{2}-\frac{1}{p}}\|\theta^*\|$. Moreover, since $\ell$ is Lipschitz w.r.t. $\|\cdot\|$, we can see it is $L$-Lipschitz w.r.t $\|\cdot\|_2$ as $\|\nabla \ell(w, x)\|_2\leq \|\nabla \ell(w, x)\|_*\leq L$. Moreover since $\ell$ is $\beta$-smooth w.r.t $\|\cdot\|$, we have $\|\nabla \ell(w, x)-\nabla \ell(w', x)\|_2\leq \|\nabla \ell(w, x)-\nabla \ell(w', x)\|_2\|_*\leq \beta \|w-w'\|\leq \beta\|w-w'\|_2$, indicating that it is $\beta$-smooth w.r.t. $\|\cdot\|_2$. Thus, we have 
\begin{equation}
   \mathbb{E}[\mathcal{L}(\hat{w})]-\mathcal{L}(\theta^*)\leq  O\left(Ld^{1-\frac{2}{p}}\|\theta^*\|^2\left(\frac{1}{\sqrt{n}}+\frac{\sqrt{d\log(1/\delta)}}{\epsilon n}\right)\right). 
\end{equation}
\end{proof}
\subsection{Proof of Theorem \ref{thm:9}}
\begin{proof}
We first recall the following lemma:
\begin{lemma}\cite{feldman2022hiding}
For a domain $\mathcal{D}$, let $\mathcal{R}^{(i)}: f\times \mathcal{D}\rightarrow \mathcal{S}^{(i)}$ for $i \in [n]$ be a sequence of algorithms such that $\mathcal{R}^{(i)}(z_{1:i-1},\cdot)$ is a $(\epsilon_0,\delta_0)$-DP local randomizer for all values of auxiliary inputs $z_{1:i-1}\in \mathcal{S}^{(1)}\times \cdots \times \mathcal{S}^{(i-1)}$. Let $\mathcal{A}_{\mathcal{S}}:\mathcal{D}^n \rightarrow \mathcal{S}^{(1)}\times \cdots \times \mathcal{S}^{(n)}$ be the algorithm that given a dataset $x_{1:n\in\mathcal{D}^n}$, sample a uniformly random permutation $\pi$, then sequentially computes $z_i = \mathcal{R}^{(i)}(z_{1:i-1},x_{\pi(i)})$ for $i\in [n]$, and the outputs $z_{1:n}$. Then for any $\delta \in [0,1]$ such that $\epsilon_0\leq \log \left(\frac{n}{16\log (2/\delta)}\right)$, $\mathcal{A}_{\mathcal{S}}$ is $(\epsilon, \delta+O(e^{\epsilon}\delta_0 n))$-DP where $\epsilon = O\left((1-e^{-\epsilon_0})\cdot (\frac{\sqrt{e^{\epsilon_0}\log (1/\delta)}}{\sqrt{n}}+\frac{e^{\epsilon_0}}{n})\right)$.
\end{lemma}
Now let's get back to the proof.
Note that by the Generalized Gaussian mechanism, we can see $\mathcal{R}(x) = g_x+\mathcal{GG}_{||\cdot||_{+}}(\sigma^2)$ with $\sigma^2 =O\left(\frac{\kappa  (\beta M+\lambda)^2 \log(1/\delta_0)}{ \epsilon_0^2}\right)$ will be a $(\epsilon_0,\delta_0)$-DP local minimizer. The output could be considered as the postprocessing of the shuffled output $\mathcal{R}(x)$. Thus, the algorithm will be $(\hat{\epsilon},\hat{\delta}+O(e^{\hat{\epsilon}}\delta_0 n))$-DP where $\hat{\epsilon}=O\left((1-e^{-\epsilon_0})\cdot (\frac{\sqrt{e^{\epsilon_0}\log (1/\hat{\delta})}}{\sqrt{n}}+\frac{e^{\epsilon_0}}{n})\right)$.

Now, assume that $\epsilon_0\leq \frac{1}{2}$, then $\exists c_1>0$, s.t.,
\begin{equation*}
\begin{aligned}
\hat{\epsilon}&\leq c_1(1-e^{-\epsilon_0})\cdot \left(\frac{\sqrt{e^{\epsilon_0}\log(1/\hat{\delta})}}{\sqrt{n}}+\frac{e^{\epsilon_0}}{n}\right)\\
& \leq c_1\cdot\left( (e^{\epsilon_0/2}-e^{-\epsilon_0/2})\cdot\sqrt{\frac{\log(1/\hat{\delta})}{n}}  +\frac{e^{\epsilon_0}  -1}{n}  \right)\\
& \leq c_1 \cdot\left( \left((1+\epsilon_0)-(1-\frac{\epsilon_0}{2}  )\right) \cdot \sqrt{\frac{\log(1/\hat{\delta})}{n}}+\frac{(1+2\epsilon_0)-1}{n}\right)\\
& =c_1 \cdot \epsilon_0\cdot \left( \frac{3}{2}\sqrt{\frac{\log(1/\hat{\delta})}{n}}+\frac{2}{n}\right).
\end{aligned}
\end{equation*}
Set $\hat{\delta}=\frac{\delta}{2}$, $\delta_0 =c_2\cdot\frac{\delta}{e^{\hat{\epsilon}}n}$ for some constant $c_2>0$ and replace $\epsilon_0=\frac{c_3\cdot \kappa (\beta M+\lambda)\cdot\sqrt{\log(1/\delta_0)}}{\sigma_1}$:
\begin{equation*}
\begin{aligned}
\hat{\epsilon}&\leq c_1\cdot c_3\cdot \frac{\kappa(\beta M+\lambda)\cdot \sqrt{\log(1/\delta_0)}}{\sigma_1}\cdot \left( \frac{3}{2}\sqrt{\frac{\log(1/\hat{\delta})}{n}}+\frac{2}{n}\right)\\
& \leq O\left(\frac{\kappa (\beta M+\lambda)\cdot \sqrt{\log(1/\delta_0)\log (1/\hat{\delta})}}{\sigma_1 \sqrt{n}}\right)\\
& \leq O\left(\frac{\kappa (\beta M+\lambda)\cdot \sqrt{\log(1/\delta)\log (e^{\hat{\epsilon}}n/\delta)}}{\sigma_1 \sqrt{n}}\right).
\end{aligned}
\end{equation*}
For any $\epsilon\leq 1$, if we set $\sigma = O\left(\frac{\kappa (\beta M+\lambda)\sqrt{\log(1/\delta)\log(n/\delta)}}{\epsilon \sqrt{n}}\right)$, then we have $\hat{\epsilon}\leq \epsilon$. Furthermore, we need $\epsilon_0 = O\left(\frac{\kappa (\beta M+\lambda)\sqrt{\log(1/\delta_0)}}{\sigma}
\right)\leq \frac{1}{2}$, which would be ensured if we set $\epsilon = O\left(\sqrt{\frac{\log(n/\delta)}{n}}\right)$. This implies that for $\sigma=O\left(\frac{\kappa (\beta M+\lambda)\cdot \log (n/\delta)}{\epsilon\sqrt{n}}\right)$, algorithm \ref{alg4} satisfies $(\epsilon,\delta)$-DP as long as $\epsilon=O\left(\sqrt{\frac{\log(n/\delta)}{n}}\right)$.
\end{proof}
\subsection{Proof of theorem \ref{thm:10}}
\begin{proof}
Denote $y_t = \frac{1}{|B_t|}\sum_{x\in B_t} g_x$, $z_t =\frac{1}{|B_t|}\sum_{x\in B_t}Z_x^t$ and  $\tilde{y}_t= y_t+z_t$.  The optimality condition for $w_t=\underset{w\in\mathcal{C}}{\arg\min} \left\{ \langle \frac{\sum_{x\in B_t} g_x + Z_x^t}{|B_t|},w\rangle +\gamma_t \cdot D_{\Phi}(w,w_{t-1})\right\}$ has the form:
\begin{equation*}
\langle \tilde{y}_t +\gamma_t (\nabla \Phi(w_t)-\nabla \Phi(w_{t-1})),z-w_t\rangle \geq 0, \forall z \in \mathcal{C}.
\end{equation*}
Equivalently, we have 
\begin{equation*}
\begin{aligned}
\langle \tilde{y}_t, w_t -z\rangle & \leq \gamma_t\langle \nabla \Phi (w_t)-\nabla \Phi (w_{t-1}), z-w_t\rangle \\
& = \gamma_t(D_{\Phi}(z,w_{t-1})-D_{\Phi}(z,w_{t})-D_{\Phi}(w_t, w_{t-1})), ~\forall z\in\mathcal{C}.
\end{aligned}
\end{equation*}
Let $\xi_t =y_t-\nabla \mathcal{L}(w_{t-1})+z_t=\tilde{y}_t-\nabla \mathcal{L}(w_{t-1})$, then we have 
\begin{equation*}
\langle \nabla \mathcal{L}(w_{t-1}),w_t -z\rangle \leq \gamma_t(D_{\Phi}(z,w_{t-1})-D_{\Phi}(z,w_{t})-D_{\Phi}(w_t, w_{t-1}))-\langle \xi_t,w_t-z\rangle.
\end{equation*}
On the other hand, we know that 
\begin{align}
\mathcal{L}(w_t) -\mathcal{L}(z)& =( \mathcal{L}(w_t) - \mathcal{L}(w_{t-1}))+(\mathcal{L}(w_{t-1})-\mathcal{L}(z))\notag\\
& =\langle \nabla \mathcal{L}(w_{t-1}),w_t-w_{t-1}\rangle +\beta \cdot D_{\Phi}(w_t,w_{t-1})+\langle \nabla \mathcal{L}(w_{t-1}), w_{t-1}-z\rangle\label{eq8}\\
& \leq \langle \nabla\mathcal{L}(w_{t-1}),w_t -z\rangle +\frac{\gamma_t}{2} D_{\Phi}(w_t, w_{t-1})\label{eq9}\\
& \leq \gamma_t (D_{\Phi}(z, w_{t-1})-D_{\Phi}(z, w_t)-\frac{1}{2} D_{\Phi}(w_t, w_{t-1}))-\langle \xi_t, w_t-z\rangle,\notag
\end{align}
where Eq. (\ref{eq8}) uses the fact that $D_{\Phi}(w_t,w_{t-1})\geq \frac{1}{2}||w_t-w_{t-1}||^2$ and $\mathcal{L}$ is smooth as well as the convexity of $\mathcal{L}$ while Eq. (\ref{eq9}) is because $\gamma_t \geq 2\beta$.

Due to the strong convexity of $D_{\Phi}(\cdot, w_{t-1})$, we have 
\begin{align*}
&\langle \xi_t, w_{t-1}-w_t\rangle \leq \frac{\gamma_t\|w_{t-1}-w_t\|_2^2}{4}+ \frac{||\xi_t||_{*}^2}{\gamma_t}\\
\implies  & \langle \xi_t, w_{t-1}-w_t\rangle \leq \frac{\gamma_t}{2}D_{\Phi}(w_t, w_{t-1})+ \frac{||\xi_t||_{*}^2}{\gamma_t}\\
\implies &\langle \xi_t, z-w_t\rangle -\frac{\gamma_t}{2}D_{\Phi}(w_t, w_{t-1})\leq \langle \xi_t, z-w_{t-1}\rangle +\frac{||\xi_t||_{*}^2}{\gamma_t}.
\end{align*}
Thus,
\begin{equation*}
\begin{aligned}
&\mathcal{L}(w_t)-\mathcal{L}(z)\leq \gamma_t (D_{\Phi}(z, w_{t-1})-D_{\Phi}(z, w_{t}))-\langle \xi_t, w_{t-1}-z\rangle +\frac{||\xi_t||_{*}^2}{\gamma_t}\\
\Rightarrow&\frac{1}{\gamma_t}(\mathcal{L}(w_t)-\mathcal{L}(z))\leq D_{\Phi}(z, w_{t-1})-D_{\Phi}(z, w_{t})-\frac{\langle \xi_t, w_{t-1}-z\rangle }{\gamma_t}+\frac{||\xi_t||_{*}^2}{\gamma_t^2}.\\
\end{aligned}
\end{equation*}

Thus, summing over $t=1,\cdots, T$, 
\begin{equation*}
\begin{aligned}
&\sum_{t=1}^T (\gamma_t^{-1})\cdot (\mathcal{L}(w_t)-\mathcal{L}(z))\leq D_{\Phi}(z, w_0)-D_{\Phi}(z, w_T)+\sum_{t=1}^T \left(
\frac{\langle \xi_t, z-w_{t-1} \rangle}{\gamma_t}+\frac{||\xi_t||_{*}^2}{\gamma_t^2}
\right)\\
\Rightarrow&(\sum_{t=1}^T \gamma_t^{-1})\cdot (\mathcal{L}(\frac{\sum_{t=1}^T \gamma_{t}^{-1} w_t}{\sum_{t=1}^T \gamma_{t}^{-1} })  -\mathcal{L}(z))\leq D_{\Phi}(z, w_0)-D_{\Phi}(z, w_T)+\sum_{t=1}^T \left(
\frac{\langle \xi_t, z-w_{t-1} \rangle}{\gamma_t}+\frac{||\xi_t||_{*}^2}{\gamma_t^2}
\right)\\
\Rightarrow&(\sum_{t=1}^T \gamma_t^{-1})\cdot (\mathcal{L}(\hat{w})  -\mathcal{L}(z))\leq D_{\Phi}(z, w_0)+\sum_{t=1}^T \left(
\frac{\langle \xi_t, z-w_{t-1} \rangle}{\gamma_t}+\frac{||\xi_t||_{*}^2}{\gamma_t^2}\right).
\end{aligned}
\end{equation*}
Take the expectation over the randomness of the noise, we get
\begin{equation*}
(\sum_{t=1}^T \gamma_t^{-1})\cdot (\mathbb{E}[\mathcal{L}(\hat{w})]-\mathcal{L}(z))\leq D_{\Phi}(z, w_0)+\sum_{t=1}^T 
\frac{\mathbb{E}[\langle \xi_t, z-w_{t-1} \rangle]}{\gamma_t}+\sum_{t=1}^T\frac{\mathbb{E}[||\xi_t||_{*}^2]}{\gamma_t^2}.
\end{equation*}
To bound the term $\sum_{t=1}^T\frac{\mathbb{E}[\langle \xi_t, z-w_{t-1}\rangle ]}{\gamma_t}$, let $x_t=y_t-\nabla \mathcal{L}(w_{t-1})$
and notice that
\begin{equation*}
\begin{aligned}
\sum_{t=1}^T\frac{\mathbb{E}[\langle \xi_t, z-w_{t-1}\rangle ]}{\gamma_t}
=&\sum_{t=1}^T\frac{\mathbb{E}[\langle y_t - \nabla \mathcal{L}(w_{t-1}), z-w_{t-1}\rangle]}{\gamma_t}  \\
=& \sum_{t=1}^T\frac{[\langle x_t, z-w_{t-1}\rangle]}{\gamma_t}.
\end{aligned}
\end{equation*}
We will bound $\sum_{t=1}^T \langle x_t, z-w_{t-1}\rangle =\sum_{t=1}^T \psi_t$. First, we recall the following lemma proposed by \cite{nazin2019algorithms}. 

\begin{lemma}\label{lemma:a17}
When $\beta M \leq \lambda$, we have
\begin{equation*}
||x_t||_{*}\leq 2\beta M+\lambda \leq 3\lambda \Rightarrow |\langle x_t, z-w_{t-1}\rangle |\leq 3\lambda M,
\end{equation*}
\begin{equation*}
||\mathbb{E}[x_t]||_{*}\leq \beta \cdot M \cdot \left(\frac{\sigma}{\lambda}\right)^2 +\frac{\sigma^2}{\lambda}\leq \frac{2\sigma^2}{\lambda}\Rightarrow |\mathbb{E}[\langle x_t, z-w_{t-1}\rangle ]|\leq \frac{2\sigma^2 M}{\lambda },
\end{equation*}
\begin{equation*}
\left(\mathbb{E}[||x_t||_{*}^2]\right)^{1/2} \leq \sigma +\beta M \cdot \frac{\sigma}{\lambda}\leq 2\sigma\Rightarrow \left(\mathbb{E}[(\langle x_t, z-w_{t-1}\rangle )^2]\right)^{1/2} \leq 2\sigma M.
\end{equation*}

\end{lemma}
Next, we recall Bernstein's inequality for martingales \cite{freedman1975tail},

\begin{lemma}
    Suppose $X_1, \cdots, X_n$ are a sequence of random variables such that $0\leq X_i\leq 1$. Define the martingale difference sequence $\{Y_n=\mathbb{E}[X_n|X_1, \cdots, X_{n-1}]-X_n\}$ and denote $K_n$ the sum of the conditional variances
\begin{equation*}
    K_n=\sum_{t=1}^n \text{Var}(X_n|X_1,\cdots, X_{n-1}).
\end{equation*}
Let $S_n=\sum_{i=1}^n X_i$, then for all $\epsilon, k\geq 0$ we have
\begin{equation}
  \text{Pr}[\sum_{i=1}^n \mathbb{E}[X_n|X_1, \cdots, X_{n-1}]-S_n\geq \epsilon, K_n\leq k]\leq \exp(-\frac{\epsilon^2}{2k+2\epsilon/3}). 
\end{equation}
\end{lemma}

we have 
\begin{equation*}
\begin{aligned}
\text{Pr}\left\{\sum_{t=1}^T \psi_t\geq \frac{2 TM\sigma^2}{\lambda}+3\cdot (2\sigma M)\sqrt{\tau T}\right\}&\leq \exp\left\{-\frac{9\cdot \tau}{2+\frac{2}{3}\cdot \frac{3\sqrt{\tau}\cdot(3\lambda M)}{2\sigma M \sqrt{T}}}\right\}\\
&\leq \exp\left\{-\frac{9\tau}{2+\frac{3\lambda \sqrt{\tau}}{\sigma \sqrt{T}}}\right\}\\
&
\leq e^{-\tau}
\end{aligned}
\end{equation*}
for all $\tau= O\left(\frac{\sigma^2 T}{\lambda^2}\right)$.

Thus, for all  $\tau= O\left(\frac{\sigma^2 T}{\lambda^2}\right)$ w. p. $1-e^{-\tau}$,
\begin{equation*}
\sum_{t=1}^T \psi_t \leq O\left(\frac{TM\sigma^2}{\lambda}+\sigma M\sqrt{T\tau}\right).
\end{equation*}
Next we bound the term of $\sum_{t=1}^T \mathbb{E}[||\xi_t||_{*}^2]$. It is notable that $$\mathbb{E}[||\xi_t||_{*}^2]= \mathbb{E}[\|x_t+z_t\|^2_*]\leq 2\|x_t\|_*^2+2\mathbb{E}[\|z_t\|_*^2]=2\|x_t\|_*^2+2g^2,$$ with $$g^2=O(\frac{1}{|B_t|} \frac{\log(\frac{n}{\delta})\cdot d\kappa (\beta M+\lambda )^2 \cdot \log(1/\delta)}{n\epsilon^2})= O(\frac{\log(\frac{n}{\delta})\cdot dT\kappa (\beta M+\lambda )^2 \cdot \log(1/\delta)}{n^2\epsilon^2}).$$  Thus, it is sufficient for us to bound $\sum_{i=1}^T \|x_t\|_*^2=\sum_{i=1}^T \phi_i$. Similar to Lemma \ref{lemma:a17} we have the following result 
\begin{lemma}\cite{nazin2019algorithms} 
When $M \leq \lambda$, we have
\begin{align*}
    &\mathbb{E}[\phi_i]\leq (\sigma+\frac{M\sigma}{\lambda})^2\leq 4\sigma^2, \\ 
    & \phi_i\leq (2M+\lambda)^2\leq 9\lambda^2,\\
    & [\mathbb{E}(\phi_i^2)]^\frac{1}{2}\leq (\sigma+\frac{M\sigma}{\lambda})(2M+\lambda)\leq 6\lambda \sigma. 
\end{align*}

\end{lemma}
Thus, by Berstern's inequality, we have if $\tau= O\left(\frac{\sigma^2 T}{\lambda^2}\right)$ 
\begin{align*}
    \text{Pr}[\sum_{t=1}^T ||x_t||_{*}^2\geq 4\sigma^2 T+18\lambda \sigma \sqrt{T\tau}  ]\leq \exp(-\frac{9\tau}{2+\frac{3\sqrt{\tau}\lambda}{\sigma\sqrt{T}}})\leq \exp(-\tau). 
\end{align*}

In total, let $\gamma_t=\bar{\gamma}$, we have with probability at least $1-2\exp(-\tau)$ 
\begin{equation}\label{aeq:28}
\mathbb{E}[\mathcal{L}(\hat{w})]-L(\theta^*)\leq O\left(
\frac{D_{\Phi}(\theta^{*},w_0)\cdot \bar{\gamma}}{T}+\frac{M\sigma^2}{\lambda}+\frac{\sigma M\sqrt{\tau}}{\sqrt{T}} +\frac{\sigma^2}{\bar{\gamma}} +\frac{M\sigma \sqrt{\tau}}{\sqrt{T}\bar{\gamma}}
+\frac{\log(\frac{n}{\delta})\cdot dT\kappa (\beta M+\lambda )^2 \cdot \log(1/\delta)}{n^2\epsilon^2\bar{\gamma} }
\right).
\end{equation}
Let $\frac{\bar{\gamma}}{T} =O(\frac{(\beta M+\lambda) \sqrt{ d \log (1/\delta)}}{nM\epsilon} )$, and since $D_{\Phi}(\theta^{*},w_0)=\Phi(\theta^*)\leq \frac{\kappa M^2}{2}$ we have
\begin{equation*}
\mathbb{E}[\mathcal{L}(\hat{w})]-L(\theta^*) \leq \tilde{O}\left(
\frac{M\sigma^2}{\lambda}+\frac{\sigma M\sqrt{\tau}}{\sqrt{T}}+\frac{M\sigma^2}{\bar{\gamma}}+\frac{(\beta M+\lambda) M \kappa \sqrt{ d\log(1/\delta)}}{n\epsilon}
\right).
\end{equation*}
Let $\lambda =\frac{\sigma \sqrt{n\epsilon}}{\sqrt[4]{\kappa^2 d \log (1/\delta)}}\geq \max\{\beta, 1\} M$, we have
\begin{equation*}
\mathbb{E}[\mathcal{L}(\hat{w})]-\mathcal{L}(\theta^*)\leq  O\left(
\frac{M\sigma \kappa \sqrt[4]{ d \log (1/\delta)}}{\sqrt{n\epsilon}} + \frac{\sigma M \sqrt{\tau }}{\sqrt{T}}+\frac{M\sigma^2}{\bar{\gamma}}
\right).
\end{equation*}
Let $\bar{\gamma}=\sqrt{T}$, then $\sqrt{T}=O(\frac{Mn\epsilon} {(\beta M+\lambda) \sqrt{d \log (1/\delta)}})$, and it holds that
\begin{equation*}
\mathbb{E}[\mathcal{L}(\hat{w})]-\mathcal{L}(\theta^*)\leq  O\left(
\frac{M\max\{\sigma^2, \sigma\} \sqrt[4]{\kappa^2 d\log(1/\delta)}\sqrt{\log(1/\delta^{'}})}{\sqrt{n\epsilon}}
\right)
\end{equation*}
w.p. at least $1-\delta^{'}$.
\end{proof}
\subsection{Proof of Theorem \ref{thm:11}}

\begin{algorithm}
	\caption{Truncated DP Batched Mirror Descent }
	\begin{algorithmic}[1]
		\State {\bfseries Input:} Dataset $D$, loss function $\ell$, initial point $w_0=0$, smooth parameter $\beta$ and  and $\lambda$. 
 \State Divide the permuted data into $T$ batches $\{B_i\}_{i=1}^T$ where $|B_i|=\frac{n}{T}$ for all $i=1,\cdots, T$
 \For{$t = 1,\cdots,T$}
\For{each $x\in B_t$}
\State \hspace{-3mm}\begin{small}$g_x= \begin{cases} \nabla \ell(w_{t-1},x)& \text{if} ~||\nabla \ell(w_{t-1},x)||_{*}\leq \beta M+\lambda\\ 0& {\text{otherwise}} \end{cases}$\end{small}
	\EndFor
 \State Let 
 \State\hspace{-2mm}\begin{small}$w_t=\underset{w\in\mathcal{\mathcal{C}}}{\arg\min} \left\{ \langle \frac{\underset{{x\in B_t}}{\sum} g_x }{|B_t|}+Z^t,w\rangle +\gamma_t \cdot D_{\Phi}(w,w_{t-1})\right\},$\end{small} where  $Z^{t}\sim \mathcal{GG}_{||\cdot||_{+}}(\sigma_1^2)$ with $\sigma_1^2 =O\left(\frac{\kappa (\beta M+\lambda )^2 \cdot \log(1/\delta)}{|B_t|^2\epsilon^2}   \right)$, $||\cdot||_{+}$ is the smooth norm for $(\mathbf{E},||\cdot||_{*})$. $\kappa = \min \{\frac{1}{p-1},\log d\}$ and  $\Phi(x) =\frac{\kappa}{2}||x||_{\kappa_{+}}^2$ with $\kappa_{+}= \frac{\kappa}{\kappa -1}$.  
	\EndFor\\
	 \Return $\hat{w} = (\sum_{t=1}^T \gamma_{t}^{-1})^{-1} \cdot \sum_{t=1}^T \gamma_{t}^{-1} w_t$
	\end{algorithmic}
	\label{alg:7}
\end{algorithm}
We propose our method in Algorithm \ref{alg:7}. Note that there are two key differences compared to Algorithm \ref{alg4}. First, since we do not need the privacy amplification via shuffling, there is no shuffling step. Secondly, instead of adding noise to each truncated gradient $g_x$, here we add a generalized Gaussian noise to the averages of the gradients for each batch. In the following we will provide our theoretical results. 
\begin{theorem}
      For the $\ell_p^d$ space with $1<p<2$, suppose Assumption \ref{ass:4} holds and assume $n$ is large enough such that $O((\frac{\sqrt{n\epsilon}M}{\kappa\sqrt[4]{d\log(1/\delta)}})^\frac{2}{3})\geq \max\{\beta, 1\}M$. For any $0<\epsilon, \delta< 1$, Algorithm \ref{alg:7} is $(\epsilon, \delta)$-DP. Moreover, if we set $\{\gamma_t\}=\gamma=\sqrt{T}$, $T=\frac{n\epsilon}{M\lambda \sqrt{d\log(1/\delta) } }$ and  $\lambda=O((\frac{\sqrt{n\epsilon}M}{\kappa\sqrt[4]{d\log(1/\delta)}})^\frac{2}{3})$. Then for any failure probability $\delta'$, the output $\hat{w}$ satisfies 
      the following with probability at least $1-\delta'$
\begin{equation*}
       \mathbb{E}[\mathcal{L}(\hat{w})]-L(\theta^*)\leq O\left(M^\frac{4}{3}\frac{\kappa^\frac{2}{3}(d\log (1/\delta))^\frac{1}{6}\sqrt{\log(1/\delta')} }{(n\epsilon)^\frac{1}{3}} \right), 
\end{equation*}
      where the expectation is taken over  the randomness of noise, and the probability is w.r.t. the dataset $D$. 
\end{theorem}
\begin{proof}
The proof of DP is just by the Generalizer Gaussian mechanism. For utility, 
the proof is almost the same as in the proof for Theorem \ref{thm:10}, while the only difference is the noise. Similar to (\ref{aeq:28}) we have the following result with probability at least $1-2\exp(-\tau)$
\begin{equation}\label{aeq:29}
\mathcal{L}(\hat{w})-\mathcal{L}(\theta^*)\leq O\left(
\frac{ \kappa M^2 \bar{\gamma}}{T}+\frac{M\sigma^2}{\lambda}+\frac{\sigma M\sqrt{\tau}}{\sqrt{T}} +\frac{\sigma^2}{\bar{\gamma}} +\frac{M\sigma \sqrt{\tau}}{\sqrt{T}\bar{\gamma}}
+\frac{ dT^2\kappa (\beta M+\lambda )^2 \cdot \log(1/\delta)}{n^2\epsilon^2\bar{\gamma} }
\right).
\end{equation}
Take $\bar{\gamma}=\sqrt{T}$ then we have 
\begin{equation*}
    \mathcal{L}(\hat{w})-\mathcal{L}(\theta^*)\leq O\left(
\frac{\kappa M^2\sqrt{\tau}}{\sqrt{T}}+\frac{M^2}{\lambda}+\frac{ dT^{3/2}\kappa \lambda^2 \cdot \log(1/\delta)}{n^2\epsilon^2 }
\right).
\end{equation*}
Take $T=\frac{n\epsilon}{M\lambda \sqrt{d\log(1/\delta) } }$ we have 
\begin{equation*}
        \mathcal{L}(\hat{w})-\mathcal{L}(\theta^*)\leq O\left(
\frac{\kappa M\sqrt{\lambda}\sqrt[4]{d\log(1/\delta)}\sqrt{\tau}}{\sqrt{n\epsilon}}+\frac{M^2}{\lambda}
\right).
\end{equation*}
    Take $\lambda=O((\frac{\sqrt{n\epsilon}M}{\kappa\sqrt[4]{d\log(1/\delta)}})^\frac{2}{3})\geq \max\{\beta, 1\}M$ we have w.p at least $1-\delta'$
\begin{equation*}
       \mathcal{L}(\hat{w})-\mathcal{L}(\theta^*)\leq O\left(M^\frac{4}{3}\frac{\kappa^\frac{2}{3}(d\log (1/\delta))^\frac{1}{6}\sqrt{\log(1/\delta')} }{(n\epsilon)^\frac{1}{3}} \right). 
\end{equation*}
\end{proof}
\subsection{Proof of Theorem \ref{thm:12}}
\cite{kamath2022improved} study DP-SCO with heavy-tailed data in Euclidean space and propose an $(\epsilon,\delta)$-DP algorithm for any $0<\epsilon, \delta<1$ that achieves an expected excess population risk of $O(M\frac{d}{\sqrt{n}}+\frac{\sqrt{M}d^\frac{5}{4}}{\sqrt{n\epsilon}})$, where $M$ is the $\ell_2$-norm diameter of the constraint set $\mathcal{C}$, under the following assumptions 
\begin{assumption}\label{ass:5}
    1) The loss function $\ell(w, x)$ is non-negative, differentiable and convex for all $w\in \mathcal{C}$. 2) The loss function is $\beta$-smooth. 3) The gradient of $\mathcal{L}(w)$ at the optimum is zero. 4) There is a constant $\sigma$ such that for all $j\in [d]$ and $w\in\mathcal{C}$ we have  $\mathbb{E}[\langle \nabla \ell(w, x)-\nabla \mathcal{L}(w), e_j\rangle^2]\leq \sigma^2$, where $e_j$ is the $j$-th standard basis vector. 5) For any $w\in\mathcal{C}$, the distribution of the gradient has bounded mean, i.e., $\|\nabla \mathcal{L}(w)\|_2\leq R$. 
\end{assumption}
 For $\ell_p^d$ space, we know that $L$-Lipschitz w.r.t $\|\cdot\|$ implies $L$-Lipschitz w.r.t $\|\cdot\|_2$. Moreover, $\mathbb{E}[||\nabla \ell(w,x)-\nabla \mathcal{L}(w)||_{*}^2]\leq \sigma^2$ implies $\mathbb{E}[||\nabla \ell(w,x)-\nabla \mathcal{L}(w)||_{2}^2]\leq \sigma^2$ which indicates condition 4) in Assumption \ref{ass:5}. For the diameter, it has the diameter of $d^{\frac{1}{2}-\frac{1}{p}}M$ w.r.t $\|\cdot\|_2$. Thus we have the following result. 
 \begin{theorem}
      For the $\ell_p^d$ space with $2\leq p\leq \infty$, suppose Assumption \ref{ass:4} holds. Then the Algorithm 1 in  \cite{kamath2022improved} is $(\epsilon, \delta)$-DP for all $0<\epsilon, \delta<1$. Moreover, suppose the loss function is non-negative, there exists $R=O(1)$ such that $\|\nabla \mathcal{L}(w)\|_*\leq R$ for all $w\in\mathcal{C}$ and 3) in Assumption \ref{ass:5} holds. then the output satisfies 
\begin{equation}
    \mathbb{E}[\mathcal{L}(w)]-\mathcal{L}(\theta^*)\leq O\left(\frac{d^{\frac{3}{2}-\frac{1}{p}} }{\sqrt{n}}+\frac{d^{\frac{3}{2}-\frac{1}{2p}} }{\sqrt{n\epsilon}}\right). 
\end{equation}
 \end{theorem}

\end{proof}

\end{document}